\documentclass[10pt,twocolumn,letterpaper]{article}

\usepackage{cvpr}              %

\usepackage[utf8]{inputenc} %
\usepackage[T1]{fontenc}    %

\usepackage[pagebackref,breaklinks,colorlinks]{hyperref}
\usepackage{url}            %
\usepackage{booktabs}       %
\usepackage{amsfonts}       %
\usepackage{nicefrac}       %
\usepackage{microtype}      %
\usepackage{graphicx}
\usepackage{color}
\usepackage[accsupp]{axessibility}
\usepackage{colortbl}
\usepackage{dsfont}
\usepackage{algorithm}
\usepackage{algorithmic}
\usepackage{amsmath,amssymb}
\usepackage{multirow}
\usepackage{wrapfig}
\usepackage{pifont}%

\usepackage{mathtools}
\usepackage{amsthm}

\theoremstyle{plain}
\newtheorem{theorem}{Theorem}
\newtheorem{definition}{Definition}
\newtheorem{proposition}{Proposition}
\newtheorem{remark}{Remark}

\usepackage[capitalize]{cleveref}
\Crefname{section}{Section}{Sections}
\Crefname{table}{Table}{Tables}

\def\cvprPaperID{12157} %
\def\confName{CVPR}
\def\confYear{2023}

\begin{document}

\title{Generalist: Decoupling Natural and Robust Generalization}

\newcommand*{\affaddr}[1]{#1} %
\newcommand*{\affmark}[1][]{\textsuperscript{#1}}

\author{%
Hongjun Wang\affmark[1]\footnotemark[1]\quad Yisen Wang\affmark[1,2]\footnotemark[2] \\
\affaddr{\affmark[1] National Key Lab of General Artificial Intelligence \\ School of Intelligence Science and Technology, Peking University}\\
\affaddr{\affmark[2] Institute for Artificial Intelligence, Peking University}
}

\maketitle

\begin{abstract}
Deep neural networks obtained by standard training have been constantly plagued by adversarial examples. Although adversarial training demonstrates its capability to defend against adversarial examples, unfortunately, it leads to an inevitable drop in the natural generalization. To address the issue, we decouple the natural generalization and the robust generalization from joint training and formulate different training strategies for each one. Specifically, instead of minimizing a global loss on the expectation over these two generalization errors, we propose a bi-expert framework called \emph{Generalist} where we simultaneously train base learners with task-aware strategies so that they can specialize in their own fields. The parameters of base learners are collected and combined to form a global learner at intervals during the training process. The global learner is then distributed to the base learners as initialized parameters for continued training. Theoretically, we prove that the risks of Generalist will get lower once the base learners are well trained. Extensive experiments verify the applicability of Generalist to achieve high accuracy on natural examples while maintaining considerable robustness to adversarial ones. Code is available at \url{https://github.com/PKU-ML/Generalist}.
\end{abstract}

\renewcommand{\thefootnote}{\fnsymbol{footnote}} 
\footnotetext[1]{Work was done as an internship at Peking University. Now, he is a Ph.D. student at the University of Hong Kong.} 
\footnotetext[2]{Corresponding Author: Yisen Wang (yisen.wang@pku.edu.cn)} 

\section{Introduction}
\label{sec:intro}

Modern deep learning techniques have achieved remarkable success in many fields, including computer vision \cite{DBLP:conf/nips/KrizhevskySH12,DBLP:conf/cvpr/HeZRS16}, natural language processing \cite{DBLP:conf/nips/VaswaniSPUJGKP17,DBLP:conf/naacl/DevlinCLT19}, and speech recognition \cite{DBLP:conf/interspeech/SakSRB15,wang2017residual}. 
Yet, deep neural networks (DNNs) suffer a catastrophic performance degradation by human imperceptible adversarial perturbations where wrong predictions are made with extremely high confidence \cite{DBLP:journals/corr/SzegedyZSBEGF13,DBLP:journals/corr/GoodfellowSS14,wang2020transferable}.
The vulnerability of DNNs has led to the proposal of various defense approaches \cite{DBLP:conf/ndss/Xu0Q18,wang2020hamiltonian,DBLP:conf/iclr/QinFSRCH20,DBLP:conf/sp/PapernotM0JS16,bai2019hilbert} for protecting DNNs from adversarial attacks. 
One of those representative techniques is adversarial training (AT) \cite{DBLP:conf/iclr/MadryMSTV18,wang2019dynamic,wang2020improving,mo2022adversarial}, which dynamically injects perturbed examples that deceive the current model but preserve the right label into the training set.
Adversarial training has been demonstrated to be the most effective method to improve adversarially robust generalization \cite{athalye2018obfuscated,wu2020adversarial}.

\begin{figure}[!t]
    \centering
    \includegraphics[width=0.4\textwidth]{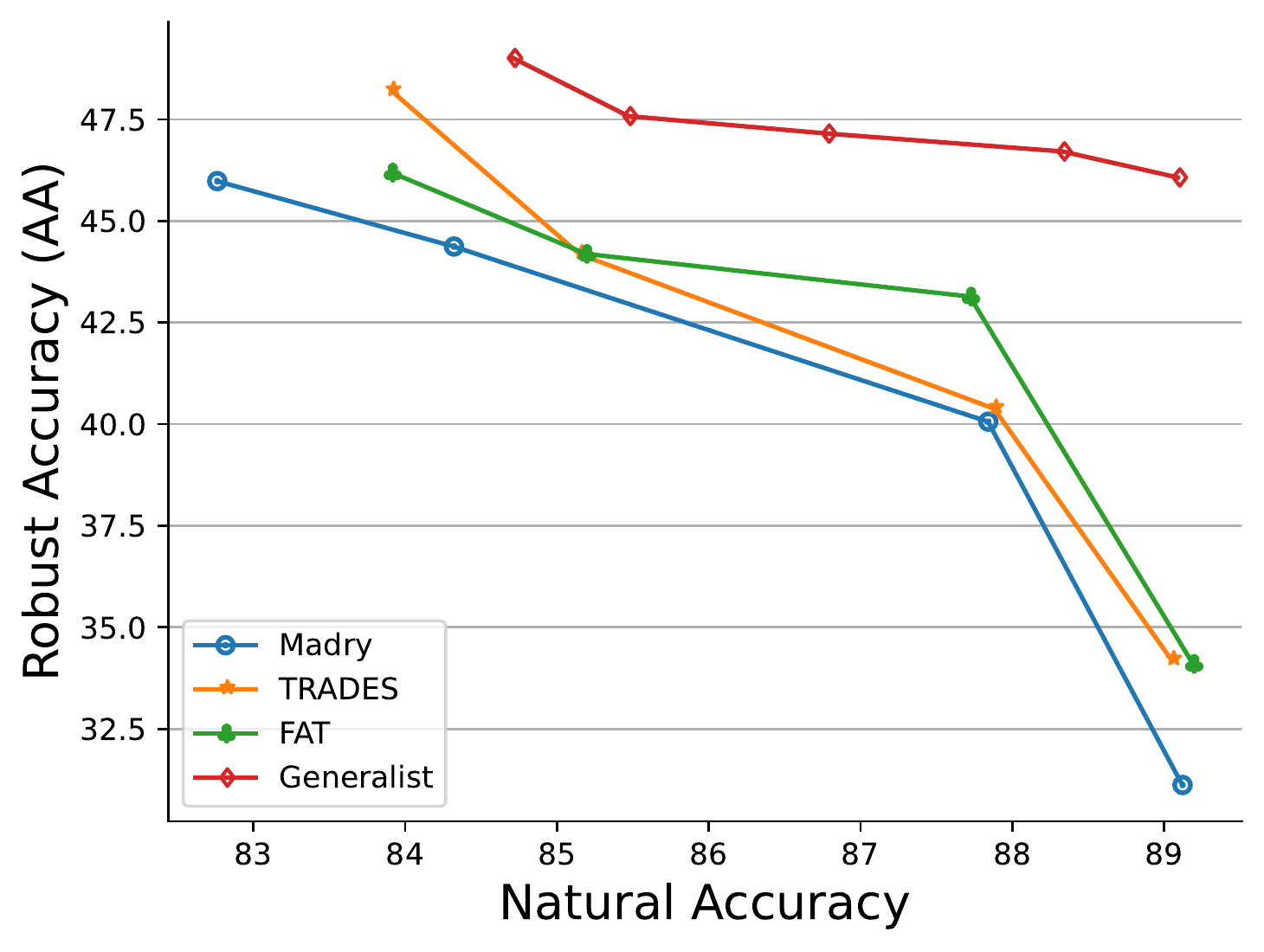}
    \caption{Comparison with other advanced adversarial training methods. Both clean accuracy and robust accuracy (against AutoAttack \cite{DBLP:conf/icml/Croce020a}) are given for a handy reference. It is noted that current adversarial training methods achieve high clean accuracy by greatly sacrificing robustness. That means it is hard to obtain sufficient robustness but maintain high clean accuracy in the joint training framework. Our Generalist attains excellent clean accuracy while staying competitively robust. The improvement of Generalist is notable since we only use the naive cross-entropy loss with negligible computational overhead and even  without increasing the model size.}\label{fig:first}
\end{figure}

Despite these successes, such attempts of adversarial training have found a tradeoff between natural and robust accuracy, \textit{i.e.}, there exists an undesirable increase in the error on unperturbed images when the error on the worst-case perturbed images decreases, as illustrated in Figure \ref{fig:first}.
Prior works \cite{DBLP:conf/iclr/TsiprasSETM19,DBLP:conf/icml/ZhangYJXGJ19} even argue that natural and robust accuracy are fundamentally at odds, which indicates that a robust classifier can be achieved only when compromising the natural generalization.
However, the following works found that the tradeoff may be settled in a roundabout way, such as incorporating additional labeled/unlabeled data \cite{DBLP:conf/nips/AlayracUHFSK19,DBLP:conf/nips/NajafiMKM19,DBLP:conf/nips/CarmonRSDL19,DBLP:conf/icml/RaghunathanXYDL20} or relaxing the magnitude of perturbations to generate suitable adversarial examples for better optimization \cite{DBLP:conf/icml/ZhangXH0CSK20,DBLP:conf/cvpr/LeeLY20}. These works all focus on the data used for training while we propose to tackle the tradeoff problem from the perspective of the \textbf{training paradigm} in this paper. 

Inspired by the spirit of the divide-and-conquer method, we decouple the objective function of adversarial training into two sub-tasks: one is used for natural example classification while the other one is used for adversarial example classification. Specifically, for each sub-task, we train a base learner on natural/adversarial datasets with the task-specific configuration while sharing the same model architecture. The parameters of base learners are collected and combined to form a global learner
at intervals during the training process, which is then distributed to base learners as initialized parameters for continued training. We name the framework as \emph{Generalist} whose proof-of-concept pipeline is shown in Figure \ref{fig:pipeline}. Different from the traditional joint training framework for natural and robust generalization, our proposed Generalist fully leverages task-specific information to individually train the base learners, which makes each sub-task to be solved better. Theoretically, we show that if the base learners are well trained, the final global learner is guaranteed to have a lower risk. Our proposed Generalist is the first to effectively address the tradeoff between natural and robust generalization by utilizing task-aware training strategies to achieve high clean accuracy in the natural setting, while also maintaining considerable robustness to the adversarial setting (as shown in Figure \ref{fig:first}). 

In summary, the main contributions are as follows:
\begin{itemize}
  \item For the tradeoff between natural and robust generalization, previous methods have struggled to find a sweet point to meet both goals in the joint training framework. Here, we propose a novel Generalist paradigm, which constructs multiple task-aware base learners to respectively achieve the generalization goal on natural and adversarial counterparts separately.
  \item For each task, rather than being constricted in a stiff manner, every detail of the training strategies (\textit{e.g.}, optimization scheme) can be totally customized, thus each base learner can better explore the optimal trajectory in its field while the global learner can fully leverage the merits of all base learners.
  \item We conduct extensive experiments in common settings against a wide range of adversarial attacks to demonstrate the effectiveness of our approach. Results show that our Generalist paradigm greatly improves both clean and robust accuracy on benchmark datasets compared to relevant techniques.
\end{itemize}

\begin{figure}[!t]
    \centering
    \includegraphics[width=0.4\textwidth]{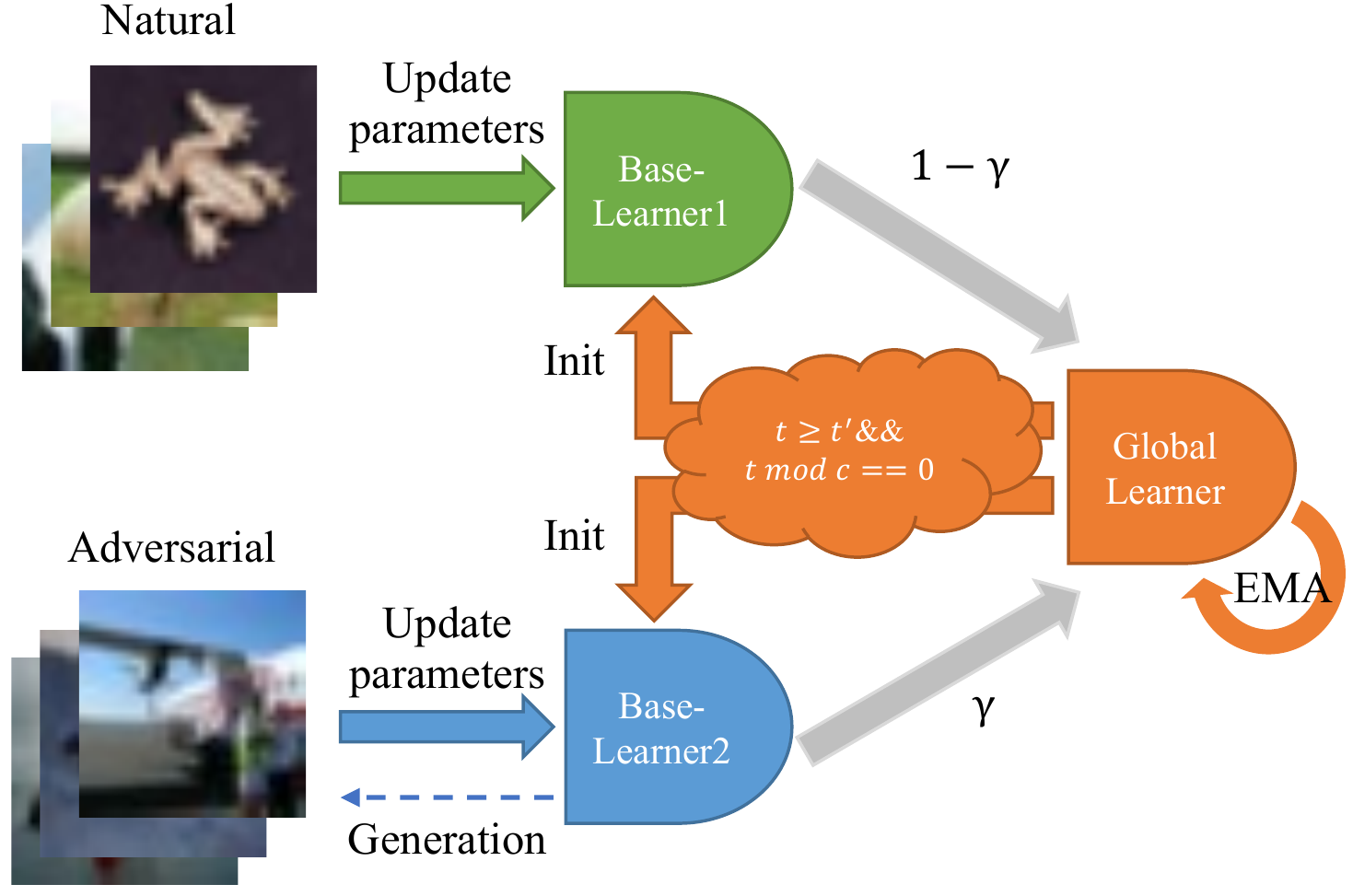}
    \caption{A pipeline for the proposed Generalist. It consists of two base learners separately trained within their respective fields and a global learner aggregates the parameters of base learners through the training process. The global learner assigns its accumulated knowledge to each base learner with a fixed frequency, based on which the base learner continues learning.} \label{fig:pipeline}
\end{figure}

\section{Preliminaries and Related Work}
\label{sec:related}
In this section, we briefly introduce some relevant background knowledge and terminology about adversarial training and meta-learning.
\label{sec:preliminaries}

\textbf{Notations.} Consider an image classification task with input space $\mathcal{X}$ and output space $\mathcal{Y}$. Let $x \in\mathcal{X} \subseteq \mathbb{R}^{d}$ denote a natural image and $y\in\mathcal{Y}=\{1,2, \ldots, K\}$ denote the corresponding ground-truth label. The natural and adversarial datasets $\mathcal{X}\times\mathcal{Y}=\left\{\left(x_i, y_i\right)\right\}_{i=1}^{n}$ and $\mathcal{X}^{\prime}\times\mathcal{Y}=\left\{\left(x^{\prime}_i, y_i\right)\right\}_{i=1}^{n}$ are sampled from a distribution $\mathcal{D}_1$ and $\mathcal{D}_2$, respectively. We denote a DNN model as $f_{\theta}: \mathcal{X} \rightarrow \mathbb{R}^{K}$ whose parameters are $\boldsymbol\theta \in \Theta$, which should classify any input image into one of $K$ classes. The objective functions $\ell_1$ and $\ell_2$ for the natural and adversarial setting can be defined as: $\ell_{1}\overset{def}{=}\mathcal{D}_{1}\times\Theta\rightarrow [0, \infty)$ and $\ell_{2}\overset{def}{=}\mathcal{D}_{2}\times\Theta\rightarrow [0, \infty)$, which are usually positive, bounded, and upper-semi continuous \cite{DBLP:journals/mor/BlanchetM19,Villani2003TopicsIO,DBLP:conf/colt/BartlettM01}. 

\subsection{Standard Adversarial Training}
The goal of the adversary is to generate a malignant example $x^{\prime}$ by adding an imperceptible perturbation $\varepsilon \in \mathbb{R}^{d}$ to $x$. And the generated adversarial example $x^{\prime}$ should be in the vicinity of $x$ so that it looks visually similar to the original one.  
This neighbor region $\mathbb{B}_{\varepsilon}(x)$ anchored at $x$ with apothem $\varepsilon$ can be defined as $\mathbb{B}_{\varepsilon}(x)=\left\{(x^{\prime},y) \in \mathcal{D}_2 \mid \left\|x-x^{\prime}\right\|_{\infty} \leq \varepsilon\right\}$.
For adversarial training, it first generates adversarial examples and then updates the parameters over these samples. The iteration process of adversarial training can be summed up as:
\begin{equation}
\small
\label{eqn:AT_basic}
\left\{\begin{array}{l}
{x^{\prime}}^{(t+1)}=\Pi_{\mathbb{B}\left(x, \epsilon\right)}\left({x^{\prime}}^{(t)}+\alpha \operatorname{sign}\left(\nabla_{x^{\prime}} \ell_2\left({x^{\prime}}^{(t)}, y; \boldsymbol\theta^t\right)\right)\right) \\
\boldsymbol\theta^{(t+1)}=\boldsymbol\theta^{(t)} - \tau\nabla_{\boldsymbol\theta} \mathbb{E}[\ell_1(x, y; \boldsymbol\theta^t)+\beta\mathcal{R}(x^{\prime}, x, y; \boldsymbol\theta^t)],
\end{array}\right.
\end{equation}
where $\Pi_{\mathbb{B}\left(x, \epsilon\right)}$ is the projection operator, $\alpha$ is the step size, $\tau$ is the learning rate, and $\mathcal{R}(\cdot)$ is the loss difference of $\ell_2(x^{\prime}, y; \boldsymbol\theta^t) - \ell_1(x, y; \boldsymbol\theta^t)$. The tradeoff factor $\beta$ balances the importance of natural and robust errors. 
Various adversarial training methods can be derived from Eq. \ref{eqn:AT_basic}. For instance, when $\beta=1$, it is equivalent to the vanilla PGD training \cite{DBLP:conf/iclr/MadryMSTV18}, and when $\beta=1/2$, it is transformed into the half-half loss in \cite{DBLP:journals/corr/GoodfellowSS14}. The formulation degenerates to standard natural training as $\beta=0$. Besides, we can get the formulation in TRADES \cite{DBLP:conf/icml/ZhangYJXGJ19} when replacing $\mathcal{R}(\cdot)$ with the KL-divergence.

\subsection{Multi-Task Learning and Meta-Initialization}

\textbf{Multi-Task Learning.} Multi-Task Learning (MTL) is to improve performance across tasks through joint training of different models \cite{DBLP:conf/nips/BilenV16,DBLP:conf/cvpr/LuKZCJF17,DBLP:conf/iclr/YangH17}. Consider a set of assignments containing data distribution and loss function defined as $\mathcal{A}=\{\mathcal{D}, \ell\}$ with corresponding models $\left\{\mathcal{M}_{a}\right\}_{a=1}^{|\mathcal{A}|}$ parameterized by trainable tensors $\boldsymbol\theta_{\mathcal{M}_{a}}$. 
In MTL, these sets have non-trivial pairwise intersections, and are trained in a joint model to find optimal parameters $\boldsymbol\theta_{\mathcal{M}_{a}}^{\star}$ for each task:
\begin{equation}
\label{eqn:MTL}
\bigcup_{a=1}^{|\mathcal{A}|} \boldsymbol\theta_{\mathcal{M}_{a}}^{\star}=\underset{\cup_{a=1}^{|\mathcal{A}|} \boldsymbol\theta_{\mathcal{M}_{a}}}{\operatorname{argmin}} \mathbb{E}_{\mathcal{A}}\mathbb{E}_{\mathcal{D}}\ \ell_{a}\left(\mathcal{D}_a; \boldsymbol\theta_{\mathcal{M}_{a}}\right),
\end{equation}
where $\ell_a(\mathcal{D}_a; \boldsymbol\theta_{\mathcal{M}_{a}})$ measures the performance of a model trained using $\boldsymbol\theta_{\mathcal{M}_{a}}$ on dataset $\mathcal{D}_a$. Our approach Generalist is directly related to MTL at first glance because both of them tend to learn a specific predictive model for different sources. However, Generalist differs significantly from MTL, \textit{i.e.}, multiple tasks are still learned \emph{jointly} under a unified form in MTL while each assignment can be optimized by heterogeneous strategies in Generalist. 

\textbf{Meta-Learning.} Meta-learning is to train a model that can quickly adapt to a new task. Suppose $\mathcal{A}$ is divided into non-overlapping splits $\mathcal{V}$ and $\mathcal{W}$, the model is first trained on the training sets and then guided by a small validation set on a set of tasks to make the trained model can be well adapted to new tasks:
\begin{equation}
\label{eqn:meta-learning}
\small
\boldsymbol\theta^{\star}=\underset{\boldsymbol\theta}{\operatorname{argmin}} \mathbb{E}_{\mathcal{V}}\mathbb{E}_{\mathcal{D}_{\mathcal{V}}}\ \ell_{\mathcal{V}}\left(\mathcal{D}_{\mathcal{V}}; \underset{\boldsymbol\theta}{\operatorname{argmin}} \mathbb{E}_{\mathcal{W}}\mathbb{E}_{\mathcal{D}_{\mathcal{W}}}\ \ell_{\mathcal{W}}\left(\mathcal{D}_{\mathcal{W}}; \boldsymbol\theta\right)\right),
\end{equation}
meta-learning \cite{DBLP:conf/icml/FinnAL17,DBLP:journals/corr/abs-1803-02999} is often designed to generalize across unseen tasks, whereas the goal of MTL is to tackle a series of known tasks. 
Nonetheless, our approach Generalist uses the technique of meta-learning to set good initializations for base learners to transfer knowledge between tasks.

\section{The Proposed New Framework: Generalist}
\label{sec:method}
Similar to a physical-world Generalist who has broad knowledge across many topics and expertise in a few, our proposed Generalist can deal with both natural and adversarial samples during test time. It consists of different base learners gradually specialized in their own disjointed fields. Over time they stretch their expertise to encompass knowledge respectively. From a starting point, Generalist takes his suitcase packed full of wide-ranging experience with him wherever it goes (\textit{i.e.}, the global learner spreads accumulated knowledge and each expert learns from the re-initialization after a certain epoch).

\subsection{Overview}
The overall procedure of our proposed algorithm is shown in Algorithm \ref{alg:Generalist}, which mainly comprises two steps: optimizing parameters of the base learner $\boldsymbol\theta_a$ in its assigned data distribution $\mathcal{D}_a$ and distributing parameters of the global learner $\boldsymbol\theta_g$ to all base learners. Base learners and the global learner share the same architecture, \textit{i.e.}, $\mathcal{M}_{1}=\mathcal{M}_{2}=\cdots=\mathcal{M}_{|\mathcal{A}|}$. Since we only focus on recognizing natural examples and adversarial ones in our setting, the total number of tasks $\mathcal{W}$ is set to two.

\begin{algorithm}[!t]
\footnotesize
   \caption{\footnotesize{Generalist: Leverage the learning trajectory with respect to task-aware base learners}}
   \label{alg:Generalist}
\algsetup{linenosize=\tiny}
\begin{algorithmic}
   \STATE {\bfseries Input:} A DNN classifier $f(\cdot)$ with initial learnable parameters $\boldsymbol\theta_g$ for the global learner and $\boldsymbol\theta_n, \boldsymbol\theta_r$ for each base learner with objective function $\ell_1, \ell_2$; number of iterations $T$; number of adversarial attack steps $K$; magnitude of perturbation $\varepsilon$; step size $\kappa$; learning rate $\tau_n, \tau_r$; exponential decay rates for ensembling $\alpha^{\prime}=0.999$; mixing ratio $\gamma$; starting point and frequency of communication $t^{\prime}, c$.
   \STATE Initialize $\boldsymbol\theta_g, \boldsymbol\theta_n, \boldsymbol\theta_r$ in $\Theta$ space.
   \FOR {t $ \leftarrow 1, 2, \cdots , T$}
    \STATE Sample a minibatch $(x, y)$ from data distribution $\mathcal{D}_1$
    \STATE \emph{/* Parallel-1: Update parameters of base learner-1 over $\mathcal{D}_1$*/}
    \STATE (Optional) Performing model ensembling, data augmentation or label smoothing, etc.
    \STATE $\boldsymbol\theta_n \leftarrow \mathcal{Z}_{n}\left[\mathbb{E}_{(x,y)}(\nabla_{\boldsymbol\theta} \ell_1(x,y; \boldsymbol\theta_{n})),\tau_{n}\right]$
    \STATE \emph{/* Parallel-2: Update parameters of base learner-2 over $\mathcal{D}_2$*/}
    \STATE $x^{\prime}_{0} \leftarrow x+\varepsilon$, $\varepsilon\sim \operatorname{Uniform}(-\varepsilon,\varepsilon)$.
    \FOR {k $ \leftarrow 1, 2, \cdots , K$}
      \STATE $x^{\prime}_{k} \leftarrow \Pi_{x^{\prime}_{k} \in \mathbb{B}_{\varepsilon}(x)}\left(\kappa \operatorname{sign}\left(x^{\prime}_{k-1}+\nabla_{x^{\prime}_{k-1}} \ell_2(x^{\prime}_{k-1}, y; \boldsymbol\theta_r)\right)\right)$
    \ENDFOR
    \STATE (Optional) Performing model ensembling, data augmentation or label smoothing, etc.
    \STATE $\boldsymbol\theta_r \leftarrow \mathcal{Z}_{r}\left[\mathbb{E}_{(x^{\prime},y)}(\nabla_{\boldsymbol\theta} \ell_2(x^{\prime}_K,y; \boldsymbol\theta_{r})),\tau_{r}\right]$
    \STATE \emph{/* For the global learner*/}
    \STATE $\boldsymbol\theta_g \leftarrow \alpha^{\prime}\boldsymbol\theta_g + (1-\alpha^{\prime})(\gamma\boldsymbol\theta_r + (1-\gamma)\boldsymbol\theta_n)$
    \IF {$t\geq t^{\prime}$ and $t\ \operatorname{mod} c==0$}
      \STATE $\boldsymbol\theta_r, \boldsymbol\theta_n \leftarrow \boldsymbol\theta_g$
    \ENDIF
   \ENDFOR
   \STATE \textbf{Return} Parameters of the global learner $\boldsymbol\theta_g$
\end{algorithmic}
\end{algorithm}

\subsection{Task-aware Base Learners}
Given a global data distribution $\mathcal{D}$ for the tradeoff problem, as denoted in Section \ref{sec:preliminaries}, $\mathcal{D}_1, \mathcal{D}_2$ are subject to the distribution of training data $\mathcal{D}_{\mathcal{W}}$. And natural images $(x,y)\sim\mathcal{D}_{1}$ while adversarial examples $(x^{\prime},y)\sim\mathcal{D}_{2}$ generated by Eq. \ref{eqn:AT_basic}. So the training process of base learners is to solve the inner minimization of Eq. \ref{eqn:meta-learning} over different distributions in a distributed manner:
\begin{equation}
\label{eqn:base learner}
\left\{\boldsymbol\theta_n^{\star},\boldsymbol\theta_r^{\star}\right\}=\underset{\bigcup_{\mathcal{W}=1}^{2}\boldsymbol\theta_{\mathcal{W}}}{\operatorname{argmin}}\mathbb{E}_{\mathcal{D}_{\mathcal{W}}}\ \ell_{\mathcal{W}}\left(\mathcal{D}_{\mathcal{W}}; \boldsymbol\theta_{\mathcal{W}}\right).
\end{equation}
Specifically, during the process, base learners $f_{\boldsymbol\theta_n}$ and $f_{\boldsymbol\theta_r}$ are assigned different subproblems that only requires accessing their own data distribution, respectively.
Note that two base learners work in a complementary manner, meaning the update of parameters is independent among base learners and the global learner always collects parameters of both base learners. So the subproblem for each base learner is defined as:
\begin{equation}
\label{eqn:subproblem}
\boldsymbol\theta_{\mathcal{W}}^{\star}=\underset{\boldsymbol\theta}{\operatorname{argmin}}{\mathcal{Z}_{\mathcal{W}}^{T}\left[\mathbb{E}_{\mathcal{W}}(\nabla_{\boldsymbol\theta} \ell_{\mathcal{W}}(\mathcal{D}_{\mathcal{W}}; \boldsymbol\theta_{\mathcal{W}})),\tau_{\mathcal{W}}\right]},
\end{equation}
where the task-aware optimizer $\mathcal{Z}_{\mathcal{W}}^{T}(\cdot,\cdot)$ search the optimal parameter states $\boldsymbol\theta_{\mathcal{W}}^{\star}$ over the subproblem $\mathcal{W}$ in $T$ rounds. Loss functions can also be task-specific and applied to each base learner separately. It is natural to consider minimizing the 0-1 loss in the natural and robust errors, however, solving the optimization problem is NP-hard thus computationally intractable. In practice, we select cross-entropy as the surrogate loss for both $\ell_1$ and $\ell_2$ since it is simple but good enough.

\subsection{Initialization from the Global Learner}
\label{sec:global}
During the initial training periods, base learners are less instrumental since they are not adequately learned. Directly initializing parameters of base learners may mislead the training procedure and further accumulate bias when mixing them. Therefore, we set aside $t^{\prime}$ epochs from the beginning for fully training base learners and just aggregates states on the searching trajectory of base learners through optimization by exponential moving average (EMA), computed as: $\boldsymbol\theta_g \leftarrow \alpha^{\prime}\boldsymbol\theta_g + (1-\alpha^{\prime})(\gamma\boldsymbol\theta_r + (1-\gamma)\boldsymbol\theta_t)$, where $\alpha^{\prime}$ is the exponential decay rates for EMA and $\gamma$ is the mixing ratio for base learners. They then learn an initialization from parameters of the global learner every $c$ epochs when each base learner is well trained in its field. Thus, the optimization of each base learner for every interlude can be expressed in Eq. \ref{eqn:opt_interlude}:
\begin{equation}
\label{eqn:opt_interlude}
\boldsymbol\theta_{\mathcal{W}}^{\star}=\underset{\boldsymbol\theta}{\operatorname{argmin}}{\mathcal{Z}_{\mathcal{W}}^{c}\left[\mathbb{E}_{\mathcal{W}}(\nabla_{\boldsymbol\theta} \ell_{\mathcal{W}}(\mathcal{D}_{\mathcal{W}}; \boldsymbol\theta_{g})),\tau_{\mathcal{W}}\right]}.
\end{equation}
Note that $\boldsymbol\theta_g$ contains both $\boldsymbol\theta_n$ and $\boldsymbol\theta_r$, meaning there always exists a term updated by gradient information of distribution different from the current subproblem. This mechanism enables fast learning within a given assignment and improves generalization, and the acceleration is applicable to the given assignment for its corresponding base learner only (proof in Appendix~\ref{apd:b1}).

With all discussed above, the learning progress of Generalist can be constructed by decending the gradient of $\boldsymbol\theta_r,\boldsymbol\theta_n$ and mixing both of them. The calculating steps in Algorithm \ref{alg:Generalist} can be summarized in Eq. \ref{eqn:overall}.
\begin{equation}
\label{eqn:overall}
\left\{\begin{array}{l}
\boldsymbol\theta_{n}^{t}=\mathcal{Z}_{n}\left[\mathbb{E}_{(x, y)\sim\mathcal{D}_1}(\nabla_{\boldsymbol\theta_n} \ell_1(x, y; \boldsymbol\theta_{n}^{t-1})),\tau_1\right] \\
\boldsymbol\theta_{r}^{t}=\mathcal{Z}_{r}\left[(\mathbb{E}_{(x^{\prime}, y)\sim\mathcal{D}_2}\nabla_{\boldsymbol\theta_r} \ell_2(x^{\prime}, y; \boldsymbol\theta_{r}^{t-1})),\tau_2\right] \\
\boldsymbol\theta_{g}^{t+1}=\alpha^{\prime}\boldsymbol\theta_g^{t-1} + (1-\alpha^{\prime})[\gamma\boldsymbol\theta_r^{t} + (1-\gamma)\boldsymbol\theta_n^{t}] \\
\boldsymbol\theta_{n}^{t}=\mathcal{B}(t, t^{\prime}, c)\boldsymbol\theta_{g}^{t+1}+(1-\mathcal{B}(t, t^{\prime}, c))\boldsymbol\theta_{n}^{t} \\
\boldsymbol\theta_{r}^{t}=\mathcal{B}(t, t^{\prime}, c)\boldsymbol\theta_{g}^{t+1}+(1-\mathcal{B}(t, t^{\prime}, c))\boldsymbol\theta_{r}^{t},
\end{array}\right.
\end{equation}
where $\mathcal{B}(t, t^{\prime}, c)$ is a Boolean function that returns one only when both $t\geq t^{\prime}$ and $t\ \operatorname{mod} c==0$, otherwise it returns zero. $\mathcal{Z}_{n}$ and $\mathcal{Z}_{r}$ are optimizers for natural training and adversarial training assignments.

\subsection{Theoretical Analysis}
In this part, we theoretically analyze how base learners help global learner in Generalist. For brevity, we omit the expectation notation over samples from each distribution without losing generalization.
\begin{definition}
(\textbf{Tradeoff Regret with Mixed Strategies}) For the natural training assignment $a_1$ and adversarial training assignment $a_2$, consider an algorithm generates the trajectory of states $\boldsymbol\theta_1$ and $\boldsymbol\theta_2$ for two base learners, the regret of both base learners on its corresponding loss function $\ell_1$, $\ell_2$ is
\begin{equation}
\mathbf{R}_{T}=\frac{1}{2} \sum_{a=1}^{2}\left(\sum_{t=1}^{T} \ell_{a}\left(\boldsymbol\theta_a^t\right)-\inf _{\boldsymbol\theta_a^t \in \Theta} \sum_{t=1}^{T} \ell_a\left(\boldsymbol\theta_{a}^{t}\right)\right).
\end{equation}
\end{definition}
The last term obtains the oracle state $\boldsymbol\theta_{a}^{\star}$, theoretically optimal parameters for each task $a$. $\mathbf{R}_{T}$ is the sum of the difference between the parameters of each base learner and the theoretically optimal parameters for each task.
Based on the definition, we can give the following upper bound on the expected error of classifier trained by Generalist with respect to $\mathbf{R}_{T}$ as:
\begin{theorem}
\label{the:2}
(Proof in Appendix~\ref{apd:b2}) Consider an algorithm with regret bound $R_{T}$ that generates the trajectory of states for two base learners, for any parameter state $\boldsymbol\theta \in \Theta$, given a sequence of convex surrogate evaluation functions ${\ell: \Theta\mapsto [0, 1]_{a\in \mathcal{A}}}$ drawn i.i.d. from some distribution $\mathcal{L}$, the expected error of the global learner $\boldsymbol\theta_{g}$ on both tasks over the test set can be bounded with probability at least $1-\delta$:
\begin{equation}
\underset{\ell \sim \mathcal{L}}{\mathbb{E}} \ell\left(\boldsymbol\theta_{g}\right) \leq \underset{\ell \sim \mathcal{L}}{\mathbb{E}} \ell\left(\boldsymbol\theta\right)+\frac{\mathbf{R}_{T}}{T}+2\sqrt{\frac{2}{T}\log \frac{1}{\delta}}.
\end{equation}
\end{theorem}
So the above inequality indicates that any strategy beneficial to reducing the error of each task that makes $\mathbf{R}_{T}$ smaller will decrease the error bound of the global learner. Considering Generalist divides the tradeoff problem into two independent tasks, Theorem \ref{the:2} guarantees the upper bound of the risks given by the global learner trained by Generalist will get lower once the error for each task becomes lower. In practice, we can apply customized learning rate strategies, optimizers, and weight averaging to guarantee the error reduction of each base learner. 
\renewcommand{\arraystretch}{0.85}
\begin{table*}[!t]
\caption{{Comparison of our algorithm with different training methods using ResNet-18 and WRN-32-10 on CIFAR-10. The maximum perturbation is $\varepsilon=8/255$. The best checkpoint is selected based on the tradeoff between clean accuracy and robust accuracy against PGD20 on the test set. We highlight the top two results on each task. We omit standard deviations of Generalist as they are very small ($<0.5\%$). Average accuracy rates (in \%) have shown that the proposed Generalist method greatly mitigates the tradeoff of the model.}}\label{tab:cifar10}
\centering
\footnotesize
\vspace{-7pt}
\begin{tabular}{lccccccccccc}
\multicolumn{12}{c}{(a) Evaluation results based on ResNet-18.}                                                                                                                                                                                                                         \\
\multicolumn{12}{l}{}                                                                                                                                                                                                                                                                       \\ \hline
\multicolumn{1}{l|}{Method}    & NAT                  & PGD20                & PGD100               & MIM                  & CW                   & $\operatorname{APGD}_{ce}$               & $\operatorname{APGD}_{dlr}$              & $\operatorname{APGD}_{t}$                & $\operatorname{FAT}_{t}$                 & Square               & AA                   \\ \hline\hline
\multicolumn{1}{l|}{NT}        & \textbf{93.04}       & 0.00                 & 0.00                 & 0.00                 & 0.00                 & 0.00                 & 0.00                 & 0.00                 & 0.00                 & 0.00                 & 0.00                 \\
\multicolumn{1}{l|}{AT ($\beta=1$)}     & 84.32                & 48.29                & 48.12                & 47.95                & 49.57                & \textbf{47.47}       & 48.57                & 45.14                & 46.17                & 54.21                & 44.37                \\
\multicolumn{1}{l|}{AT ($\beta=1/2$)}     & 87.84                & 44.51                & 44.53                & 47.30                & 44.93                & 40.58                & 42.55                & 40.20                & 44.56                & 50.76                & 40.06                \\ \hline\hline
\multicolumn{1}{l|}{TRADES ($\lambda=6$)} & 83.91                & \textbf{54.25}       & \textbf{52.21}       & \textbf{55.65}       & \textbf{52.22}       & \textbf{53.47}       & \textbf{50.89}       & \textbf{48.23}       & \textbf{48.53}       & \textbf{55.75}       & \textbf{48.20}       \\
\multicolumn{1}{l|}{TRADES ($\lambda=1$)} & 87.88                & 45.58                & 45.60                & 47.91                & 45.05                & 42.95                & 42.49                & 40.38                & 43.89                & 53.49                & 40.32                \\
\multicolumn{1}{l|}{FAT}       & 87.72                & 46.69                & 46.81                & 47.03                & 49.66                & 46.20                & 47.51                & 44.88                & 45.76                & 52.98                & 43.14                \\
\multicolumn{1}{l|}{IAT}       & 84.60                & 40.83                & 40.87                & 43.07                & 39.57                & 37.56                & 37.95                & 35.13                & 36.06                & 49.30                & 35.13                \\
\multicolumn{1}{l|}{RST}       & 84.71                & 44.23                & 44.31                & 45.33                & 42.82                & 41.25                & 42.01                & 40.41                & 46.54                & 50.49                & 37.68                \\ \hline\hline
\multicolumn{1}{l|}{Generalist}    & \textbf{89.09}       & \textbf{50.01}       & \textbf{50.00}       & \textbf{52.19}       & \textbf{50.04}       & 46.53                & \textbf{48.70}       & \textbf{46.37}       & \textbf{47.32}       & \textbf{56.68}       & \textbf{46.07}       \\ \hline
\multicolumn{12}{l}{}                                                                                                                                                                                                                                                                       \\
\multicolumn{12}{c}{(b) Evaluation results based on WRN-32-10.}                                                                                                                                                                                                                         \\
\multicolumn{12}{l}{}                                                                                                                                                                                                                                                                       \\ \hline
\multicolumn{1}{l|}{Method}    & NAT                  & PGD20                & PGD100               & MIM                  & CW                   & $\operatorname{APGD}_{ce}$               & $\operatorname{APGD}_{dlr}$              & $\operatorname{APGD}_{t}$                & $\operatorname{FAT}_{t}$                 & Square               & AA                   \\ \hline\hline
\multicolumn{1}{l|}{NT}        & \textbf{93.30} & 0.01           & 0.02           & 0.05           & 0.00           & 0.00           & 0.00           & 0.00           & 0.87           & 0.28           & 0.00           \\
\multicolumn{1}{l|}{AT ($\beta=1$)}     & 87.32          & 49.01          & 48.83          & 48.25          & 52.80          & 48.83          & 49.00          & 46.34          & 48.17         & 54.26          & 46.11          \\
\multicolumn{1}{l|}{AT ($\beta=1/2$)}      & 89.27          & 48.95          & 48.86          & 51.35          & 49.56          & 45.98          & 47.66          & 44.89          & 46.42          & 56.83          & 44.81          \\ \hline\hline
\multicolumn{1}{l|}{TRADES ($\lambda=6$)} & 85.11          & \textbf{54.58} & \textbf{54.82} & \textbf{55.67} & \textbf{54.91} & \textbf{54.89} & \textbf{55.50}  & \textbf{52.71} & \textbf{52.61} & \textbf{57.62} & \textbf{52.19} \\
\multicolumn{1}{l|}{TRADES ($\lambda=1$)} & 87.20          & 51.33          & 51.65          & 52.47          & 53.19          & 51.60          & 51.88          & 49.97          & 50.01          & 54.83          & 49.81          \\
\multicolumn{1}{l|}{FAT}       & 89.65          & 48.74          & 48.69          & 48.24          & 52.11          & 48.50          & 48.81          & 46.70          &  46.17         & 51.51          & 44.73          \\
\multicolumn{1}{l|}{IAT}       & 87.93          & 50.55          & 50.72          & 52.37          & 48.71          & 47.71          & 46.55          & 43.84          & 45.78          & 56.52          & 43.80          \\
\multicolumn{1}{l|}{RST}       & 87.27          & 46.55          & 46.76          & 47.02          & 45.99          & 45.73          & 46.58          & 45.78          & 43.18          & 52.44          & 41.52          \\ \hline\hline
\multicolumn{1}{l|}{Generalist}    & \textbf{91.03} & \textbf{56.88} & \textbf{56.92} & \textbf{58.87} & \textbf{57.23} & \textbf{53.94} & \textbf{55.80} & \textbf{53.00} & \textbf{53.65} & \textbf{63.10} & \textbf{52.91} \\ \hline
\multicolumn{12}{c}{}                    
\vspace{-11pt}
\end{tabular}
\end{table*}

\section{Experiments}

We conduct a series of experiments on ResNet-18 \cite{DBLP:conf/cvpr/HeZRS16} and WRN-32-10 \cite{DBLP:journals/corr/ZagoruykoK16} on benchmark datasets MNIST, SVHN, CIFAR-10, and CIFAR-100 under the $L_{\infty}$ norm. 

\textbf{Baselines.} We select six approaches to compare with: AT using PGD ($\beta=1$ in Eq. \ref{eqn:AT_basic}) \cite{DBLP:conf/iclr/MadryMSTV18}, AT using the half-half loss ($\beta=1/2$ in Eq. \ref{eqn:AT_basic}) \cite{DBLP:journals/corr/GoodfellowSS14}, TRADES with different $\lambda$ \cite{DBLP:conf/icml/ZhangYJXGJ19}, Friendly Adversarial Training (FAT) \cite{DBLP:conf/icml/ZhangXH0CSK20}, Interpolated Adversarial Training (IAT) \cite{DBLP:conf/ccs/LambVKB19}, and Robust Self Training (RST) \cite{DBLP:conf/icml/RaghunathanXYDL20} used labeled data for fair comparison. For Generalist, we set $t^{\prime}=75$ and the optimal mixing strategy will be discussed in Section \ref{sec:mixing}.

\textbf{Evaluation.} To evaluate the robustness of the proposed method, we apply several adversarial attacks including PGD \cite{DBLP:conf/iclr/MadryMSTV18}, MIM \cite{DBLP:conf/cvpr/DongLPS0HL18}, CW \cite{DBLP:conf/sp/Carlini017}, AutoAttack (AA) \cite{DBLP:conf/icml/Croce020a} and all its components ($\operatorname{APGD}_{ce}$, $\operatorname{APGD}_{dlr}$, $\operatorname{APGD}_t$, $\operatorname{FAB}_t$, and Square attacks).

\subsection{Tradeoff Performance on Benchmark Datasets}
\label{sec:exp}
To comprehensively manifest the power of our Generalist method, we present the results of both ResNet-18 and WRN-32-10 on CIFAR-10 in Table \ref{tab:cifar10}. 

In Table \ref{tab:cifar10}(a), Generalist consistently improves standard test error relative to models trained by several robust methods, while maintaining adversarial robustness at the same level. 
More specifically, Generalist achieves the second highest standard accuracy of 89.09\% (only lower than 93.04\% obtained by natural training (NT)), while meantime robust accuracy against AA is 46.07\%, hanging on to 48.2\% from TRADES. If we force TRADES to meet the same level of clean accuracy as Generalist (89\%), the robustness of TRADES against APGD will drop to 30\% (see TRADES in Appendix~\ref{apd:a4}), which is significantly worse than Generalist. That means it is hard to obtain acceptable robustness but maintain clean accuracy above 89\% in the joint training framework even if it is equipped with an advanced loss function, while the improvement of Generalist is notable since we only use the naive cross-entropy loss.
Contrary to FAT managing the tradeoff through adaptively decreasing the step size of PGD, which still hurts robustness a lot, Generalist is the only method with clean accuracy above 89\% and robust accuracy against AA above 46\%. We should emphasize the final obtained model of Generalist is the \emph{same size} as other trained models are. For the training time, Generalist does perform both NT and naive AT but the cost of NT is negligible, so the overhead of Generalist is smaller than TRADES, and whatever serial and parallel versions of Generalist are even \emph{faster} than TRADES (see Appendix~\ref{apd:a4}). 

Things become more obvious when it comes to WRN-32-10. In Table \ref{tab:cifar10}(b), the gap between test natural accuracy of Generalist and NT is reduced to 2.27\%, a relative decrease of 3.65\% in standard test error as compared to the second highest natural accuracy (except NT) achieved by FAT. 
It is also remarkable that the boost of accuracy does not hurt the robustness of Generalist, instead, Generalist even outperforms TRADES across multiple types of adversarial attacks. 
In particular, we find that Generalist has a standard test error of 6.7\% while TRADES with $\lambda=6$ has a standard test error of 14.89\% only. And the improved robustness of Generalist among PGD20/100, MIM, CW, $\operatorname{FAT}_{t}$ and Square is conspicuous. Besides, the best performance on AA, which is an ensemble of different attacks and the most powerful adaptive adversarial attack so far, demonstrates the reliability of Generalist.
Likewise, only Generalist attains robust accuracy of AA higher than 52\% along with clean accuracy higher than 90\%. It should be emphasized that these features confirm the practicability of Generalist.
In short, Generalist has consistently improved robustness without loss of natural accuracy. More results on benchmark datasets of MNIST, SVHN, and CIFAR-100 are in Appendix~\ref{apd:a2} and \ref{apd:a3}. 

\subsection{Comprehensive Understanding of Generalist}
We run a number of ablations to analyze the Generalist framework in this part. As illustrated in Algorithm \ref{alg:Generalist}, two factors control the tradeoff between accuracy and robustness of the global learner: \emph{frequency of communication $c$} and \emph{mixing ratio $\gamma$}. Here, we investigate how these parameters affect performance. 
If not specified otherwise, the experiments are conducted on CIFAR-10 using ResNet-18.

\subsubsection{Mixing Strategies of $\gamma$} 
In Generalist, $\gamma$ controls the tradeoff via balancing the contribution of individuals to the global learner when base learners are gradually well trained. Note that $\gamma$ is a scalar but we do not explicitly assign a fixed value to it. Instead, we set several breakpoints and dynamically adjust the value along the training process using a piecewise linear function to decrease.

Results are shown in Figure \ref{fig:c_and_m}. The numbers in brackets are the values at the 0/40/80/120-th epoch. If $\gamma$ gets smaller, the base learner in charge of natural classification has a pronounced influence on the global learner. Among all configurations, the best one is to apply $\gamma=(1,1,1,0)$ and $c=5$ to the global learner after the 75th epoch. 
When compared to strategies that $\gamma$ decays during late periods, $\gamma=(1,1,0.8,0.2)$ shows lower standard and robust accuracy, confirming that more sophisticated initialization could be useful for both accuracy and robustness.
With the increase of the last breakpoint of dynamical strategies, the robust accuracy gradually increases; while the standard accuracy decreases by a small margin. 
We also investigate the static/dynamic strategy for $\gamma$. By observing $\gamma=0.5$ and $\gamma=(1,1,1,0.5)$, the scheduled mixing strategy makes Generalist more robust to various attacks.

\subsubsection{Communication Frequency $c$} 
\label{sec:mixing}
In Generalist, $c$ controls the communication frequency between the global learner and base learners. Therefore, for $c$, with the fixed mixing ratio strategy, we sweep over the frequency of communication from 1 to 15.

Results are shown in Figure \ref{fig:c_and_m}, and we have the following observations. 
\begin{figure}[!t]
  \begin{minipage}{1.0\linewidth}
    \makebox[.5\linewidth]{\includegraphics[width=.5\linewidth]{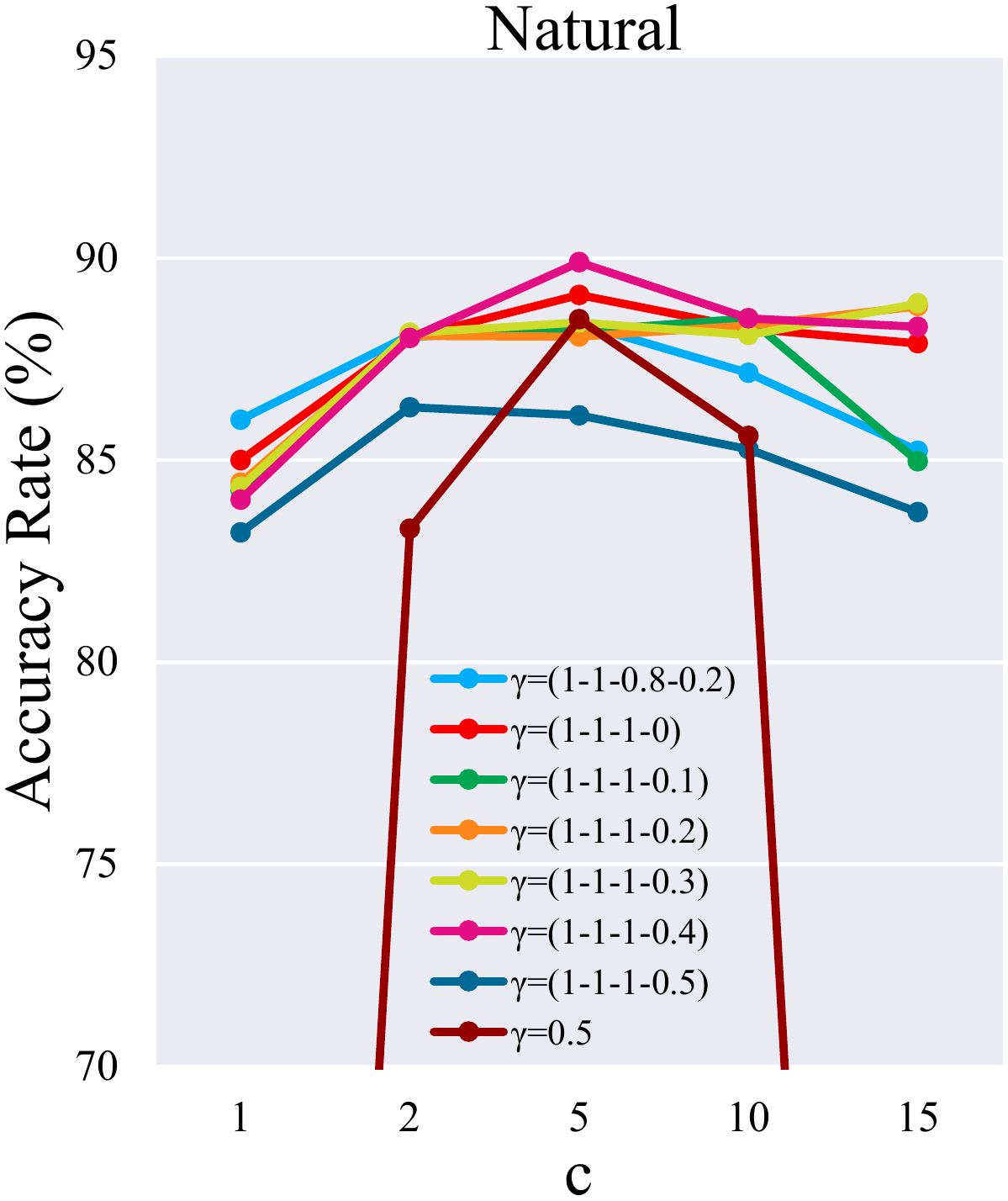}}%
    \makebox[.5\linewidth]{\includegraphics[width=.5\linewidth]{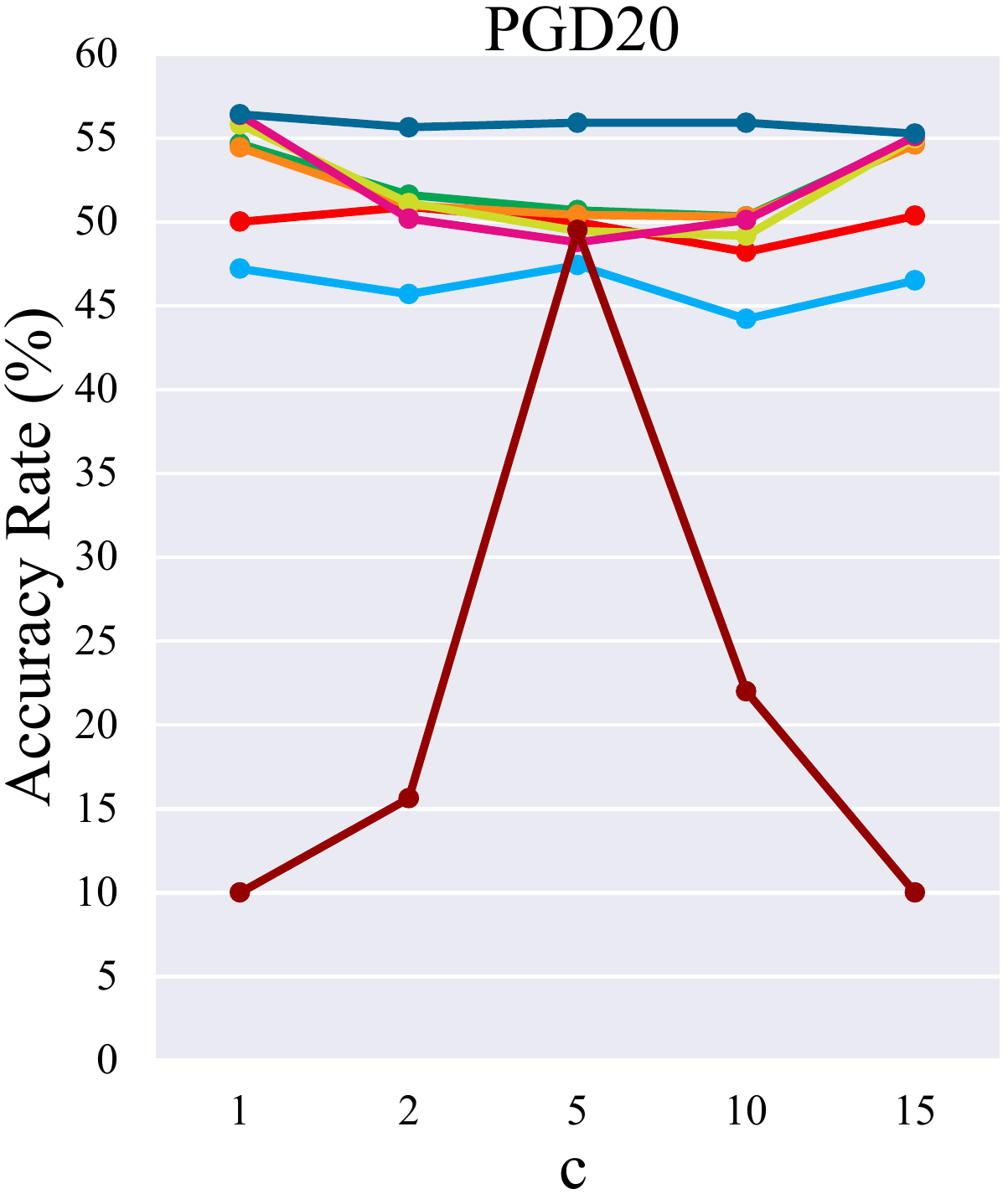}}
    \makebox[.5\linewidth]{\includegraphics[width=.5\linewidth]{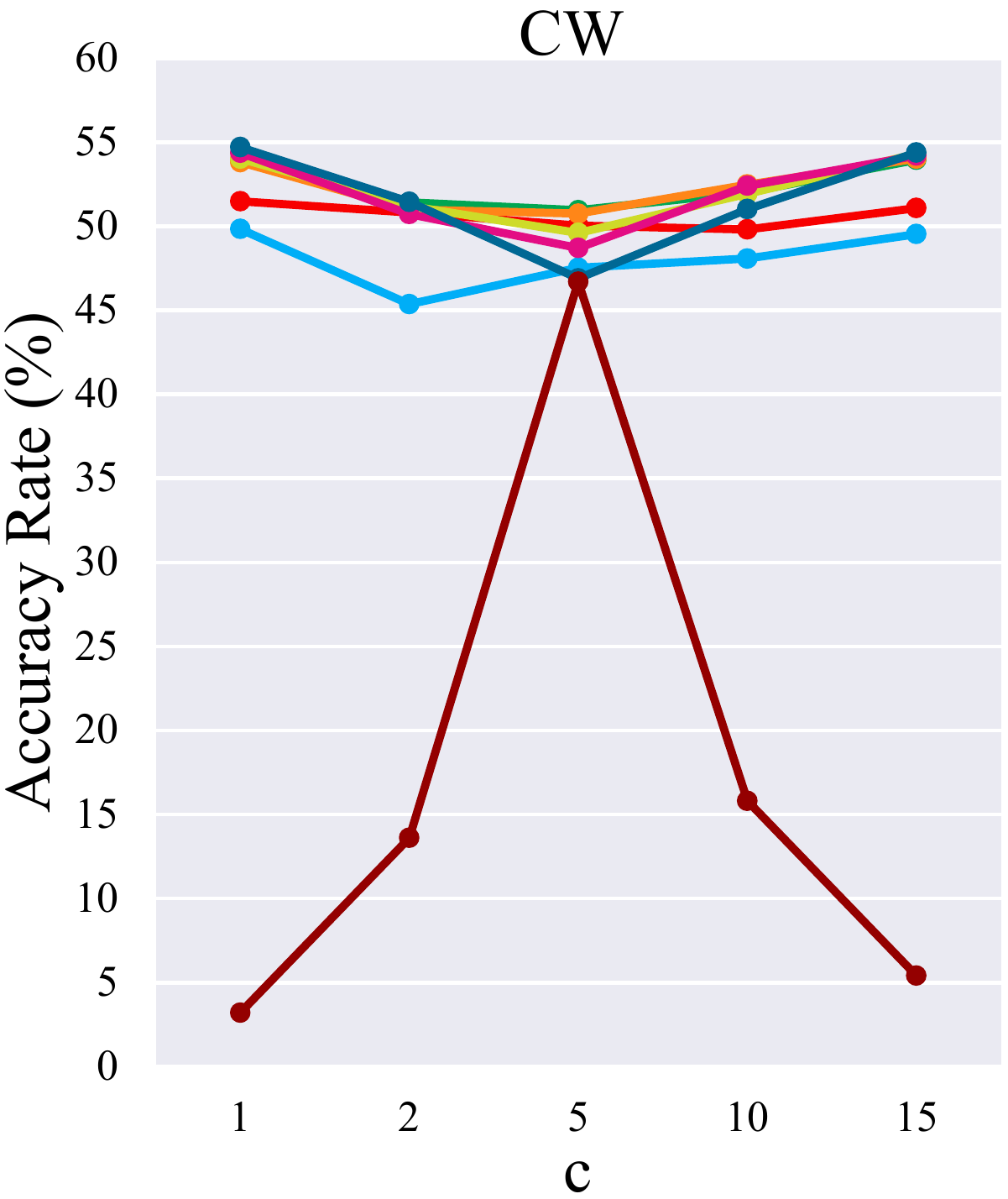}}%
    \makebox[.5\linewidth]{\includegraphics[width=.5\linewidth]{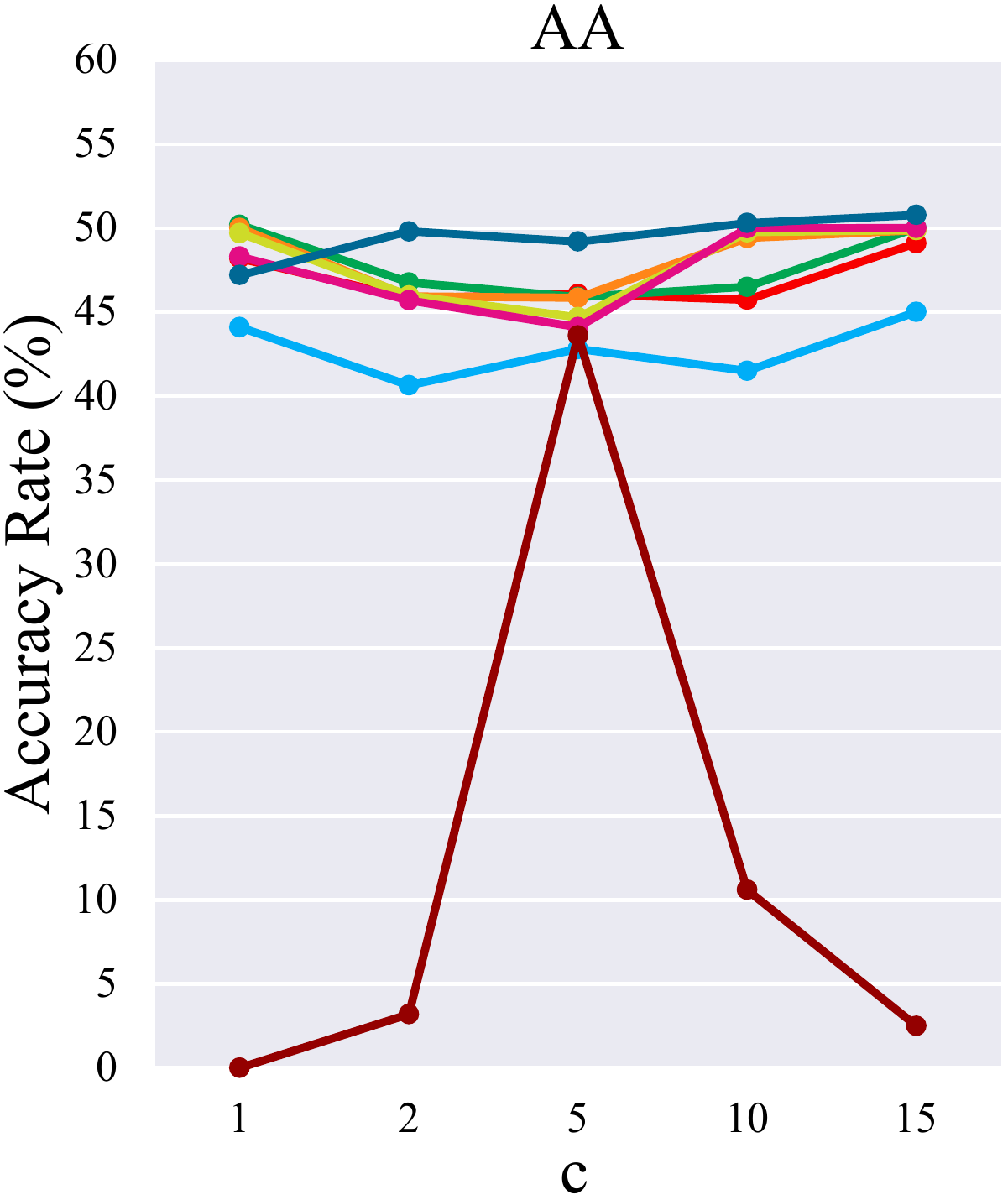}}
  \end{minipage}%
    \caption{{Generalist with different mixing ratio strategies and various values of frequency on CIFAR-10. We evaluate both natural accuracy and robustness against PGD20, C\&W and AA attacks using ResNet-18.}}\label{fig:c_and_m}
    \vspace{-11pt}
\end{figure}
Intuitively, a larger $c$ means base learners communicate with the global learner less frequently to get the initialization, so they barely have the opportunity to move alternately towards two optimal solution manifolds. 
But specifically, the natural accuracy falls back down after reaching the peak while the robust accuracy in different adversarial settings roughly shows a trough. Such observation manifests that too much/little communication has a negative influence on standard accuracy but results in relatively higher robustness. It captures a tradeoff between natural and robust errors with respect to $c$. 

\begin{figure*}[!ht]
  \begin{minipage}{1.0\linewidth}
    \makebox[.24\linewidth]{\includegraphics[width=.24\linewidth]{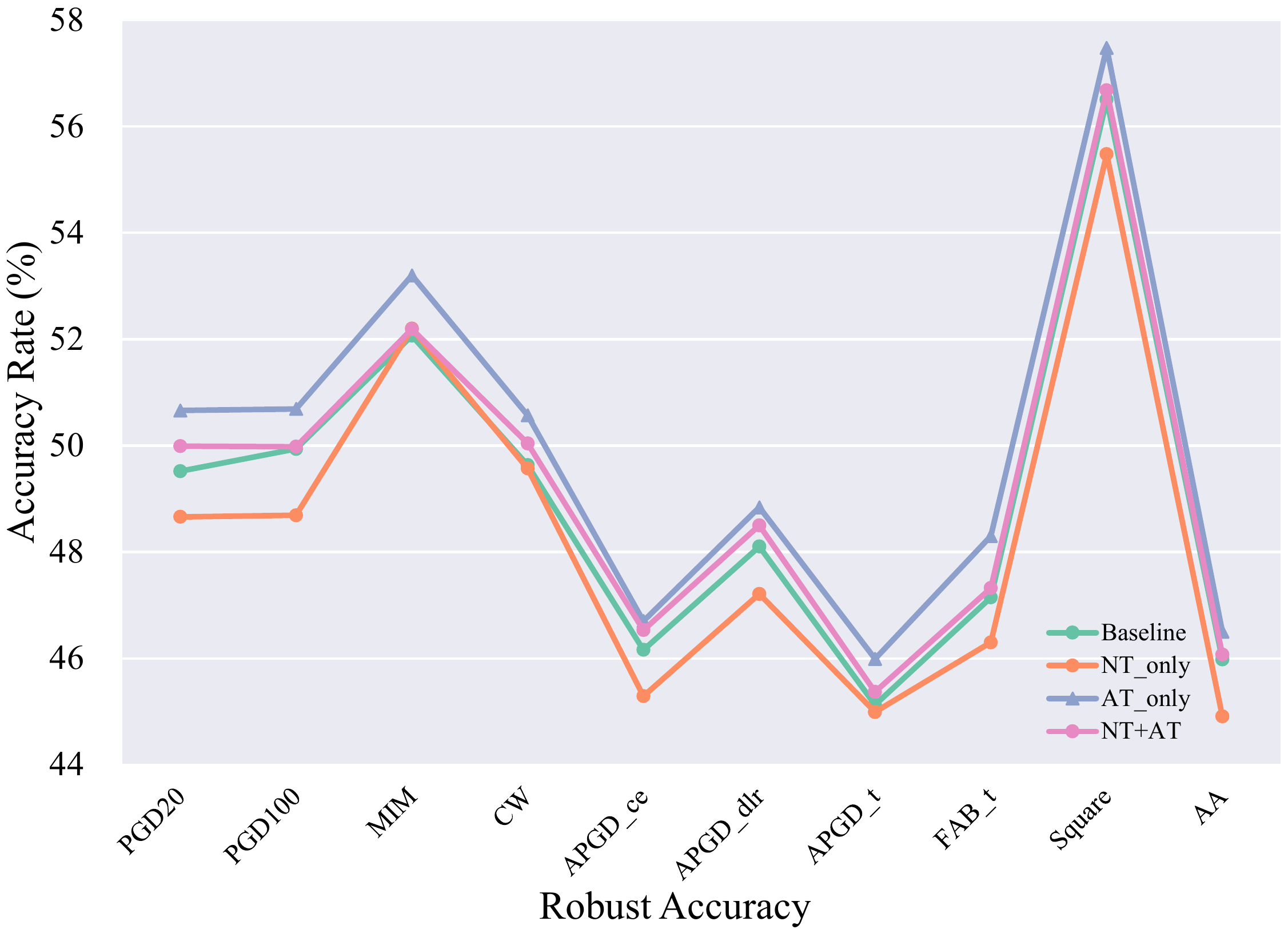}}%
    \makebox[.25\linewidth]{\includegraphics[width=.25\linewidth]{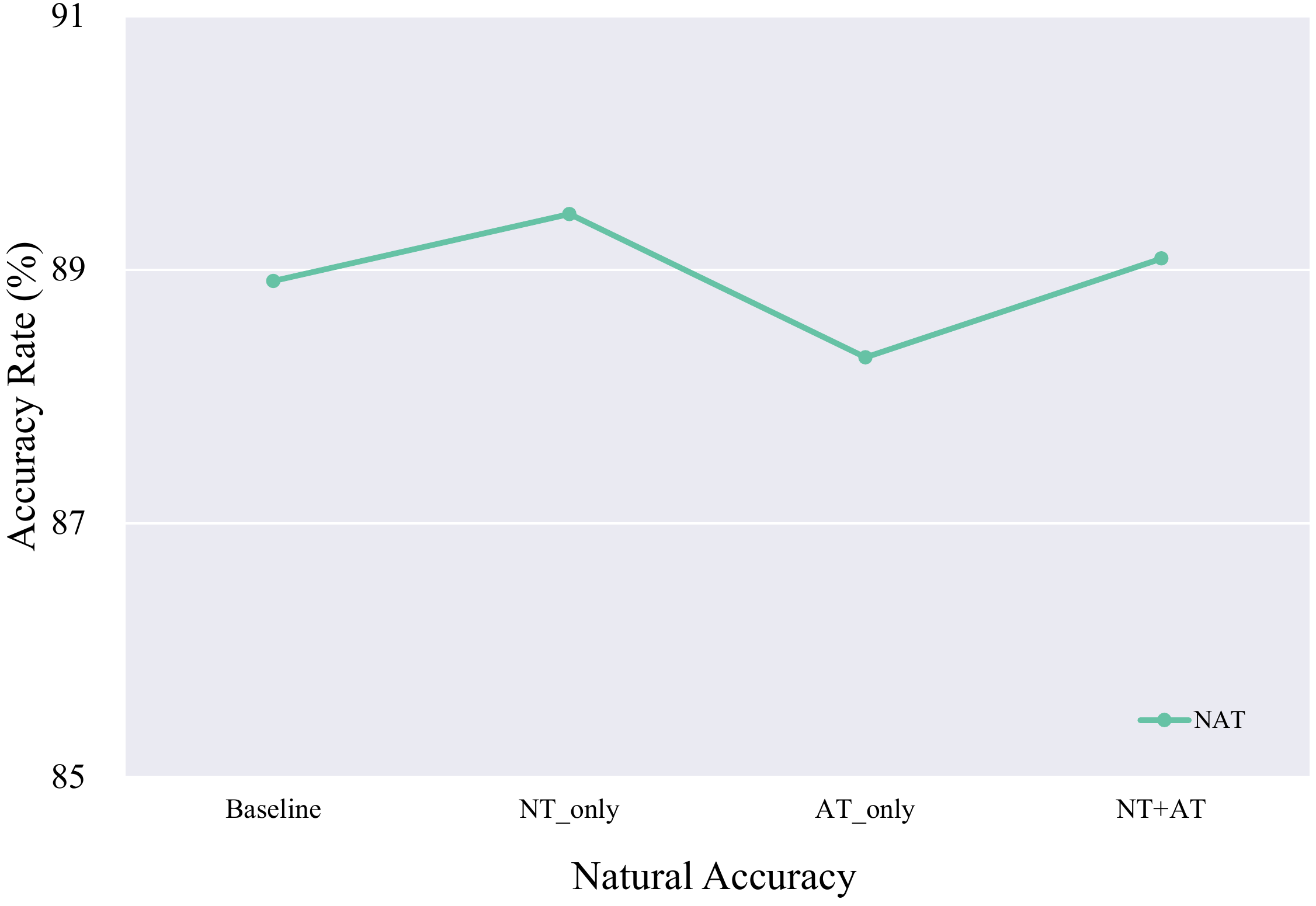}}%
    \makebox[.27\linewidth]{\includegraphics[width=.24\linewidth]{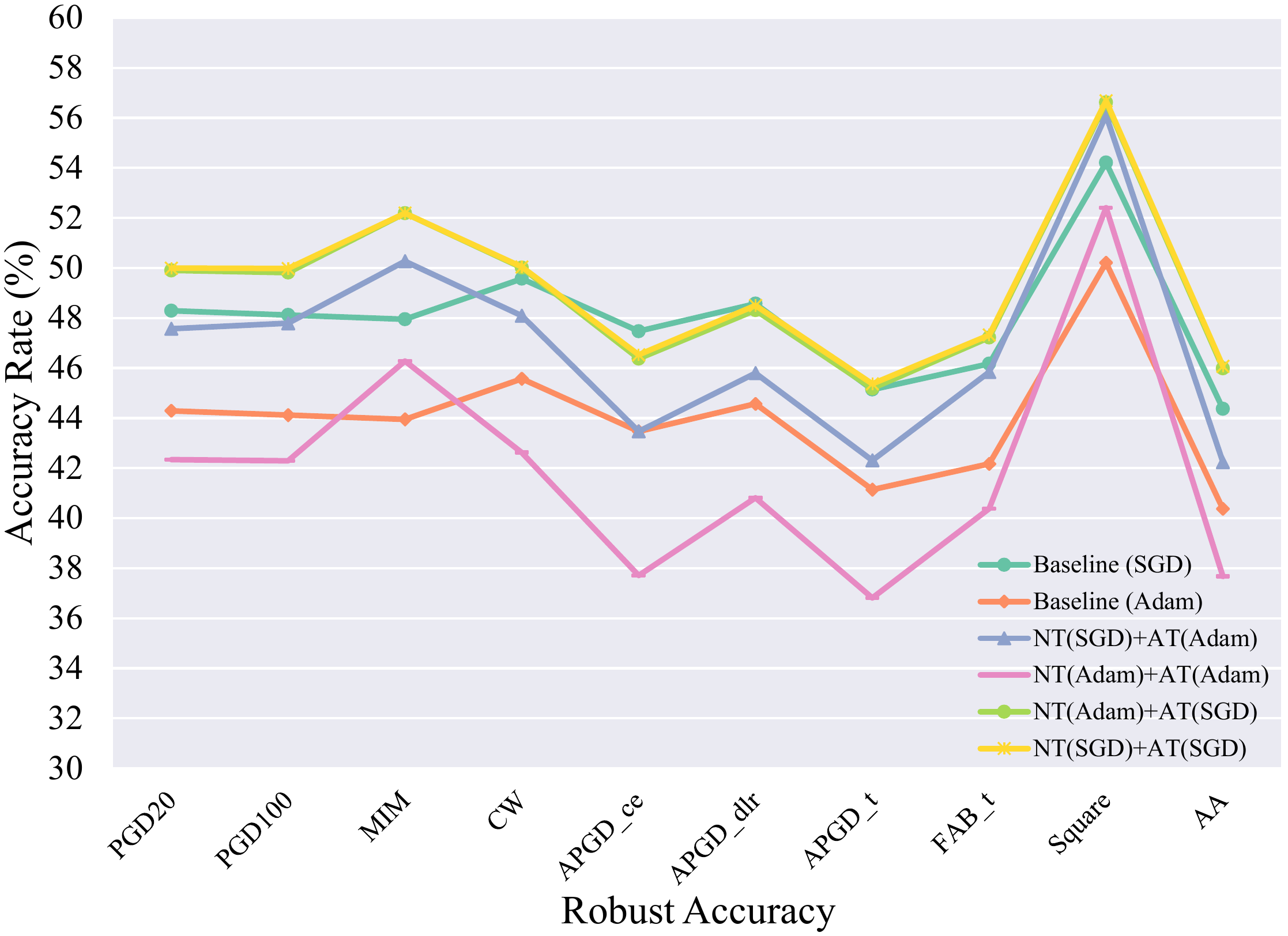}}%
    \makebox[.25\linewidth]{\includegraphics[width=.25\linewidth]{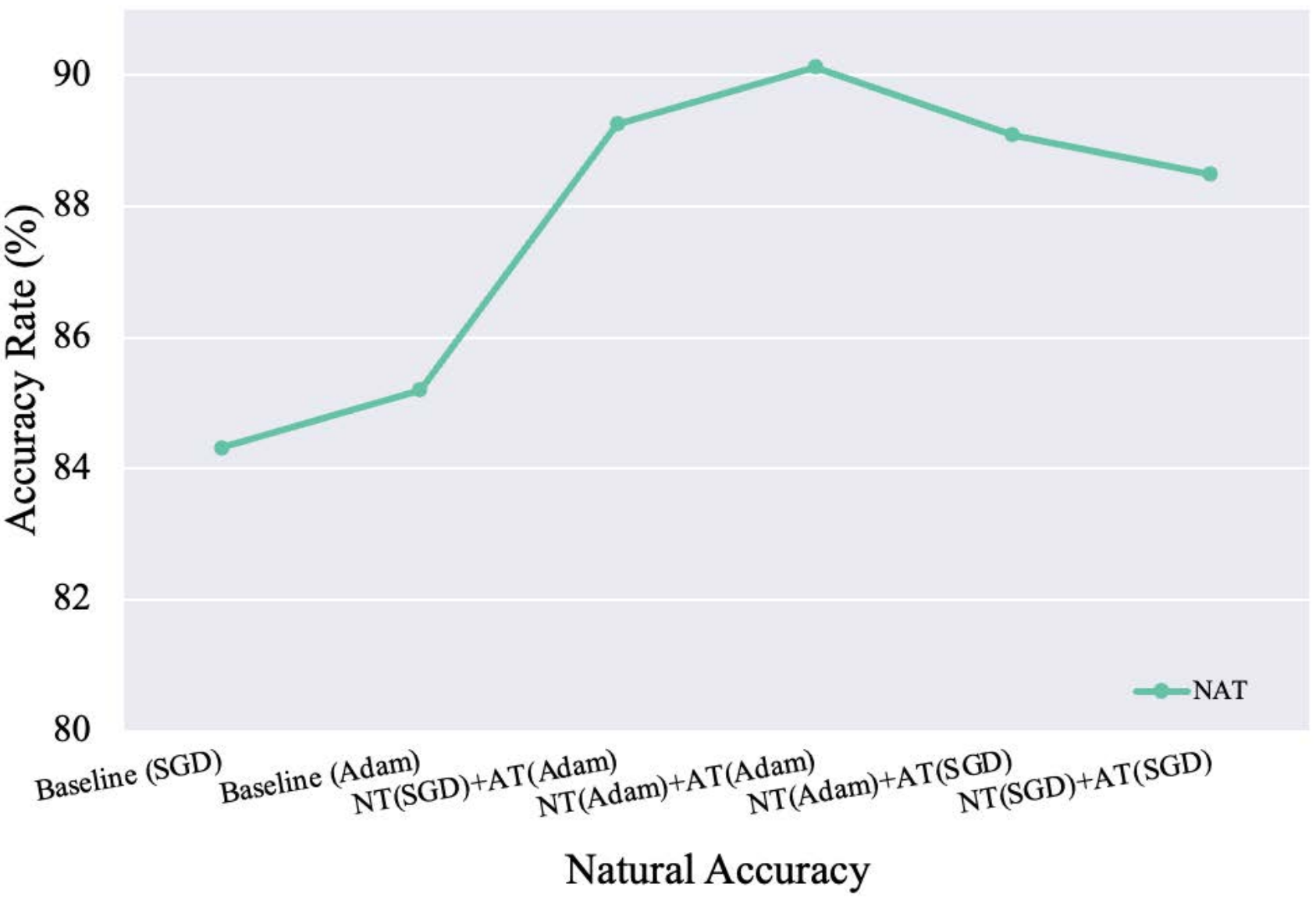}}%
    
    \makebox[.5\linewidth]{\small (a) Weight Averaging}%
    \makebox[.5\linewidth]{\small (b) Different Optimizers}%
  \end{minipage}%
\caption{{(a) We apply weight averaging to one of the base learners or both of them. Results demonstrate that using weight averaging through training can bring performance boost in its corresponding sub-task, and thus has an effect on predictions of the global learner. (b) Base learners of Generalist optimized by different optimizers. The optimal selection is using Adam for the natural classification task but maintaining SGD for the adversarial one.}}\label{fig:wa_opt}
\vspace{-7pt}
\end{figure*}

\subsubsection{Parameter Selection}

In practice, it is natural to select the mixing parameter $\gamma$ and the frequency of communication $c$ under a scenario without knowing the target model or dataset. We can find the best parameters on specific architecture and dataset, which is then transferred to others, \textit{i.e.}, choosing the best $\gamma$, $c$, and their strategies on one model/dataset and then used for other models/datasets, which still works well. Specifically, for the experimental results on MNIST/SVHN/CIFAR100 shown in Appendix~\ref{apd:a2} and \ref{apd:a3}, we just find the optimal parameters and updating strategies on CIFAR10 and apply similar $\gamma$ and $c$ on the target datasets and architectures without much fine-tuning. The performance is still very good.

\subsection{Customized Policies for Individuals} 
As stressed above, one of the major advantages of Generalist in comparison with the standard joint training framework is that each base learner enables to customize the corresponding strategy for their own tasks freely rather than using the same strategy for all tasks. In this part, we investigate whether Generalist performs better when cooperating with diverse techniques.  

\textbf{Weight Averaging.} Recent works \cite{DBLP:journals/corr/abs-2103-01946,DBLP:conf/uai/IzmailovPGVW18,wang2022selfensemble} have shown that weight averaging (WA) greatly improves both natural and robust generalization. The average parameters of all history model snapshot through the training process to build an ensemble model on the fly. However, such technique cannot benefit both accuracy and robustness in the joint training framework. Therefore, we introduce WA into base learners separately. Results are shown in Figure \ref{fig:wa_opt} (a). We employ WA in either NT (NT\_only) or AT (AT\_only) or both of them (NT+AT). 
Overall, the results confirm that the performance of the global learner can be further improved after both base learners exploit WA. 
But unfortunately, an obvious tradeoff happens if only one of the base learner is equipped with WA. For instance, the standard test accuracy of NT\_only continues to increase at the expense of the drop in the ability to defend attacks.
A likely reason is that WA implicitly controls the learning speed of base learners.
Indeed, the base learner with WA becomes an expert much faster than the one without WA in its sub-task, meaning the fast one is not in accordance with the slow one.
This result is important because it not only illustrates the potential of Generalist comes from its base learners but also identifies a key challenge of tradeoff for future improvement.

\textbf{Different Optimizers.} We also investigate the effect of optimizers designed for different tasks. We choose AT ($\beta=1$) using SGD with momentum and Adam for piecewise learning rate schedule optimized by joint training as the baseline. The initial learning rate for Adam is 0.0001. We alternately apply these two optimizers in each subproblem. The comparison of the results is shown in Figure \ref{fig:wa_opt} (b). We can see that the gap of robust accuracy between models adversarially trained by Adam and the ones trained by SGD is significant.
All three schemes equipped with Adam, namely NT (Adam)+AT (Adam), NT (SGD)+AT (Adam), and Baseline (Adam), perform worse than the ones using SGD when evaluated by adversarial attacks. 
But on the other hand, by comparing the results of Baseline (Adam) and NT (Adam)+AT (SGD), it confirms a proper optimization scheme with respect to data distribution can effectively benefit the corresponding performance without overlooking the other.
That not only demonstrates the necessity of Generalist to decouple task-aware assignments from joint training but also indicates using Adam may not be the principal reason for robustness drop. It is just ill-suited for the outer and inner optimization in AT. Besides, though the best results still come from using SGD, the learning rate for different tasks can be customized which is not feasible in the joint framework, as shown in Appendix~\ref{apd:a5}.

\begin{figure*}[!ht]
  \begin{minipage}{1.0\linewidth}
    \makebox[.5\linewidth]{\includegraphics[width=.45\linewidth]{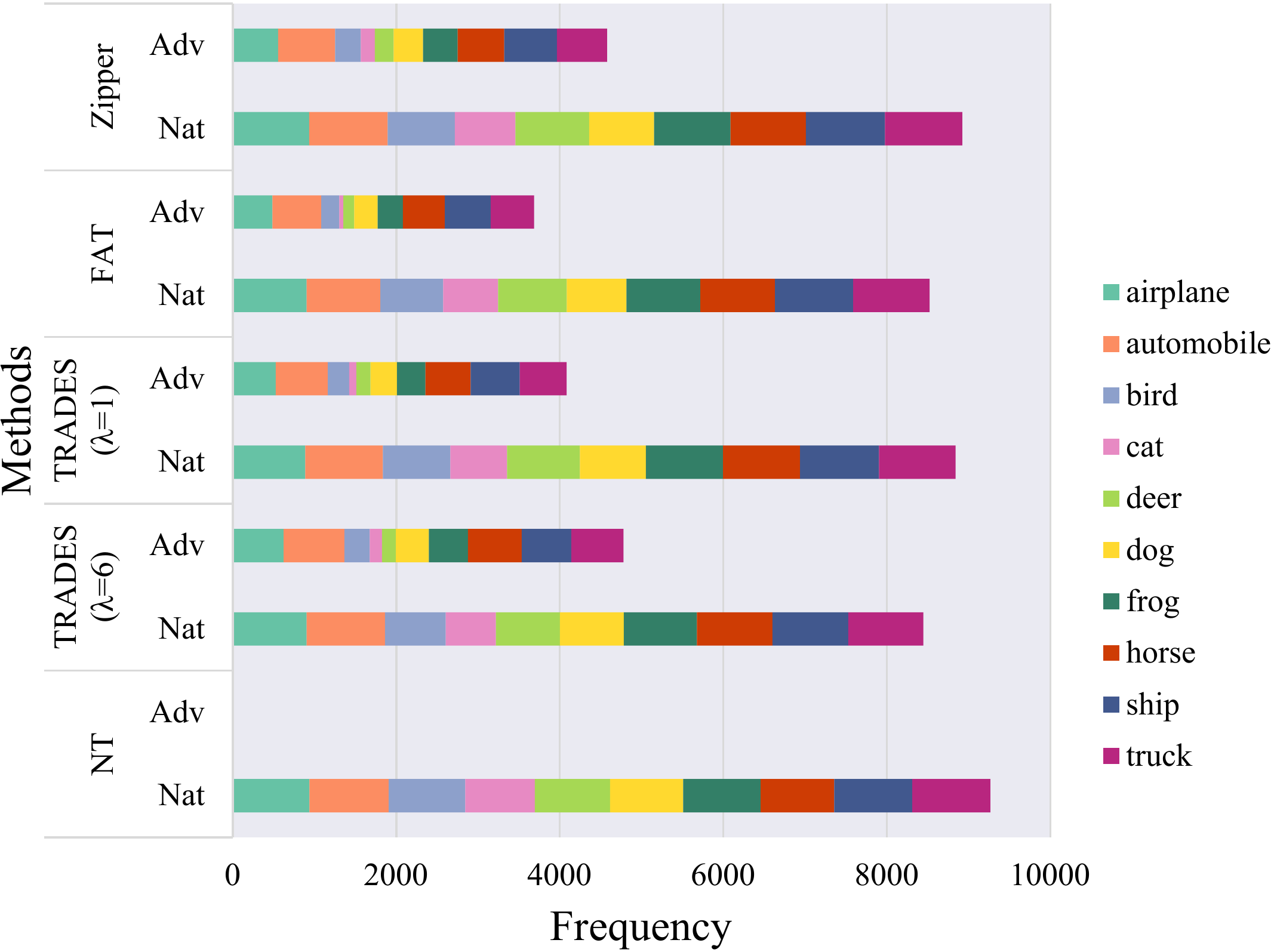}}%
    \makebox[.5\linewidth]{\includegraphics[width=.55\linewidth]{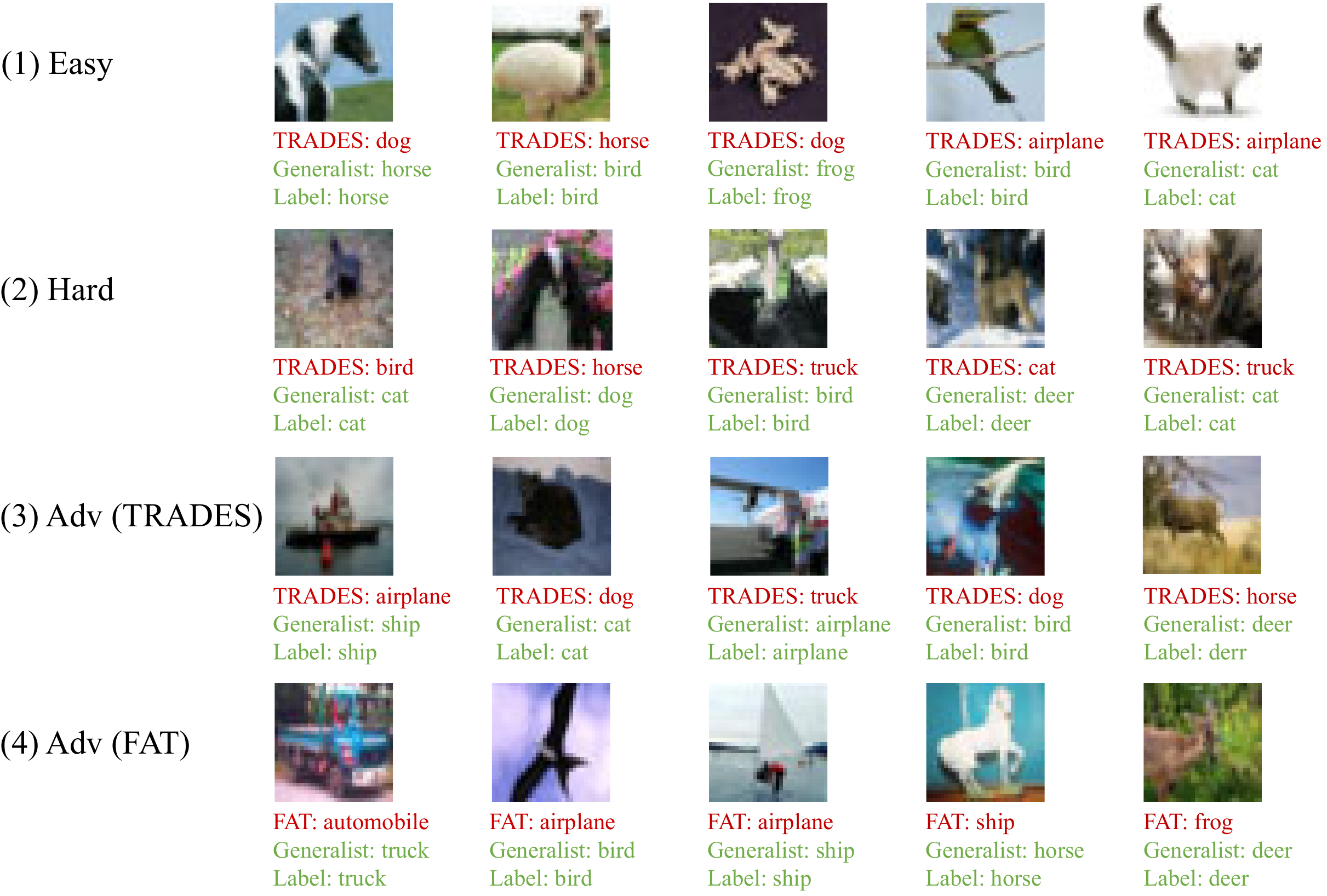}} %
    
    \makebox[.5\linewidth]{\small (a)}%
    \makebox[.55\linewidth]{\small (b)}%
  \end{minipage}%
    \caption{{Analyses of predilections that different robust classifiers have on CIFAR-10 using ResNet-18. (a) Distribution of the correct predictions of different training methods for each class. We separate out results on natural examples from adversarial ones (AA). Note that results of `NT Adv' does not appear in the figure just because they are literally zero. (b) Visualization of samples that other methods misclassify while Generalist makes right predictions. }}\label{fig:vis}
    \vspace{-11pt}
\end{figure*}

\subsection{Visualization} 
Considering the proposed method achieves impressive clean accuracy without a harsh drop in robustness, it is naturally to ask what improvements Generalist has secured in comparison with robust methods in detail.
Thus, we further investigate the predictions that robust classifiers are prone to make.
As shown in Figure \ref{fig:vis}, we provide two perspectives to analyze the differences that classifiers trained by different AT methods. 

To broadly study the case, we perform experiments on NT, TRADES with different $\lambda$, FAT and Generalist, then plot the distribution of the correct predictions of all methods for each class in Figure \ref{fig:vis}(a). 
As evident at first glance, we note that animals are more frequently misclassified, especially cats/dogs in the natural scenario and cats/deers in the adversarial scenario.
In addition, the classifier trained by standard natural training does not always outperform the ones adversarially trained. Actually, they are equally skilled at most categories and the outcome is decided by specific categories (e.g. birds, cats and dogs).
Generalist keeps pace with NT in the natural task, and meanwhile promotes the higher improvements in difficult items (e.g. cats and deers) against AA attack.

In Figure \ref{fig:vis}(b), we display specific samples in the testing dataset that are misclassified by robust classifiers (TRADES and FAT) but recognized by our proposed method, including both natural examples (the first two rows) and adversarial examples (the last two rows). 
Here, images shown in the first row are \emph{easy} ones where the foreground objects stand out from the clear backgrounds, while \emph{hard} samples are referred to those having confused objects with messy backgrounds.
It is worth noting that TRADES delivers poor performances not only on hard examples with complex backgrounds or obscured objects but also on simple ones.
For example, each image in the first row is typically plain and regular, however, TRADES fails in categorizing them into the right class. 
A plausible explanation for the issue is that TRADES lacks in a set of support measures specially devised for the natural classification task unlike Generalist does, highlighting design differentiation for sub-tasks is necessary.

Another interesting finding is that though both TRADES and FAT can build a robust classifier, they still rely on spurious background information and thus are easily deceived when encountering images with similar backgrounds but different objects. 
This phenomenon can be verified from the misclassification of the fourth and fifth images in the first row (taking white/blue backgrounds as evidence), and the fifth image in the fourth row (confused by the green background).
But Generalist has the ability to sift the invariant feature of the foreground object while ignoring the background information spuriously correlated with the categories in both natural and adversarial settings. On the whole, Generalist demonstrates its strength to differentiate difficult samples close to the decision boundary and its potential to learn a background-invariant classifier.

\section{Conclusion}
\label{conclusion}
In this paper, we propose a bi-expert framework named Generalist for improving the tradeoff issue between natural and robust generalization, which trains two base learners responsible for complementary fields and collects their parameters to construct a global learner. By decoupling from the joint training paradigm, each base learner can wield customized strategies based on data distribution. We provide theoretical analysis to justify the effectiveness of task-aware strategies and extensive experiments show that Generalist better mitigates the tradeoff of accuracy and robustness.

\section*{Acknowledgement}
Yisen Wang is partially supported by the National Key R\&D Program of China (2022ZD0160304), the National Natural Science Foundation of China (62006153), Open Research Projects of Zhejiang Lab (No. 2022RC0AB05), and Huawei Technologies Inc.

{\small
\bibliographystyle{ieee_fullname}
\bibliography{egbib}
}

\clearpage

\onecolumn

\appendix

\section{Additional Experiments}
\label{apd:a}

\setcounter{table}{1}

\subsection{Detailed Configurations}
\label{apd:a1}
All images are normalized into $[0, 1]$. We train ResNet-18 using SGD with 0.9 momentum for 120 epochs (200 epochs for CIFAR-100) and the weight decay factor is set to $3.5e^{-3}$ for ResNet-18 and $7e^{-4}$ for WRN-32-10. We use the piecewise linear learning rate strategy for performing weight averaging in base-learners. For the base-learner of AT, the initial learning rate for ResNet-18 is set to 0.01 and 0.1 for WRN-32-10 till Epoch 40 and then linearly reduced by 10 at Epoch 60 and 120, respectively. The magnitude of maximum perturbation at each pixel is $\varepsilon=8/255$ with step size $\kappa=2/255$ and the PGD steps number in the inner maximization is 10. For the base-learner of NT, we fix the initial learning rate as 0.1 and the weight decay is $5e^{-4}$ for both ResNet-18 and WRN-32-10. 

\subsection{Experiments on MNIST/SVHN}
\label{apd:a2}
We conducted experiments on MNIST ($\varepsilon=0.3$) and SVHN using ResNet-18 with the same setup in Sec. \ref{apd:a}. We ran 5 individual trials and results with standard deviations are shown in Table \ref{tab:mnist_svhn}. Our Generalist still achieves the best performance.

\renewcommand{\arraystretch}{0.85}
\begin{table}[!ht]
\caption{\footnotesize{Comparison of our algorithm with different training methods using ResNet-18 on MNIST and SVHN. The maximum perturbation is $\varepsilon=8/255$. The best checkpoint is selected based on the tradeoff between clean accuracy and robust accuracy against PGD20 on the test set. We highlight the top two results on each task. Average accuracy rates (in \%) have shown that the proposed Generalist method greatly mitigates the tradeoff of the model.}}\label{tab:mnist_svhn}
\centering
\begin{tabular}{l|ccc|ccc}
        & \multicolumn{3}{c|}{MNIST}                                                                                                                                                                                 & \multicolumn{3}{c}{SVHN}                                                                                                                                                                                     \\
Methods & NAT                                                                & PGD20                                                              & AA                                                               & NAT                                                                & PGD20                                                              & AA                                                                 \\ \hline\hline
TRADES  & \begin{tabular}[c]{@{}c@{}}99.07\\ $\pm$0.13\end{tabular}          & \begin{tabular}[c]{@{}c@{}}94.45\\ $\pm$0.07\end{tabular}          & \begin{tabular}[c]{@{}c@{}}92.17\\ $\pm$0.21\end{tabular}        & \begin{tabular}[c]{@{}c@{}}93.1\\ $\pm$0.25\end{tabular}           & \textbf{\begin{tabular}[c]{@{}c@{}}55.38\\ $\pm$0.71\end{tabular}} & \textbf{\begin{tabular}[c]{@{}c@{}}45.52\\ $\pm$0.37\end{tabular}} \\ \hline
FAT     & \begin{tabular}[c]{@{}c@{}}99.18\\ $\pm$0.03\end{tabular}          & \begin{tabular}[c]{@{}c@{}}93.54\\ $\pm$0.1\end{tabular}           & \begin{tabular}[c]{@{}c@{}}90.04\\ $\pm$0.68\end{tabular}        & \begin{tabular}[c]{@{}c@{}}93.87\\ $\pm$0.4\end{tabular}           & \begin{tabular}[c]{@{}c@{}}53.61\\ $\pm$0.88\end{tabular}          & \begin{tabular}[c]{@{}c@{}}40.92\\ $\pm$0.29\end{tabular}          \\ \hline
Generalist  & \textbf{\begin{tabular}[c]{@{}c@{}}99.24\\ $\pm$0.07\end{tabular}} & \textbf{\begin{tabular}[c]{@{}c@{}}96.14\\ $\pm$0.15\end{tabular}} & \textbf{\begin{tabular}[c]{@{}c@{}}92.3\\ $\pm$0.3\end{tabular}} & \textbf{\begin{tabular}[c]{@{}c@{}}94.11\\ $\pm$0.27\end{tabular}} & \textbf{\begin{tabular}[c]{@{}c@{}}55.29\\ $\pm$0.23\end{tabular}} & \textbf{\begin{tabular}[c]{@{}c@{}}45.41\\ $\pm$0.26\end{tabular}}
\end{tabular}
\end{table}

\subsection{Experiments on CIFAR-100}
\label{apd:a3}
To further demonstrate our proposed Generalist achieves a better tradeoff between accuracy and robustness, we also conduct experiments on CIFAR-100 datasets. Here we still use ResNet-18 as the backbone model with the same configurations as claimed in Sec. \ref{apd:a}. We report the results of natural accuracy and several advanced adversarial attack methods in Table \ref{tab:cifar100}. Note that we do not design a specialized strategy for Generalist on CIFAR-100 but Generalist still achieves a gratifying tradeoff, so it still has the potential to perform better.

\begin{table*}[!ht]
\caption{Comparison of our algorithm with different training methods using ResNet-18 on CIFAR-100. The maximum perturbation is $\varepsilon=8/255$. The best checkpoint is selected based on the tradeoff between clean accuracy and robust accuracy against PGD20 on the test set. We highlight the top two results on each task. Average accuracy rates (in \%) have shown that the proposed Generalist method greatly mitigates the tradeoff of the model.}\label{tab:cifar100}
\small
\centering
\begin{tabular}{l|ccccccccccc}
Method    & NAT                  & PGD20                & PGD100               & MIM                  & CW                   & $\operatorname{APGD}_{ce}$               & $\operatorname{APGD}_{dlr}$              & $\operatorname{APGD}_{t}$                & $\operatorname{FAT}_{t}$                 & Square               & AA                   \\ \hline\hline
NT            & \textbf{65.74} & 0.02           & 0.01           & 0.02           & 0.01           & 0.00           & 0.00           & 0.00           & 0.07           & 0.37           & 0.00           \\
AT ($\beta=1$)            & 60.10          & 28.22          & 28.27          & 28.31          & 24.87          & 26.63          & 24.13          & 21.98          & 23.91          & 27.93          & 23.87          \\
AT ($\beta=1/2$) & 60.84          & 22.64          & 22.61          & 23.86          & 22.28          & 20.66          & 21.67          & 19.2          & 20.09          & 25.36          & 19.17           \\ \hline
TRADES ($\lambda=6$)        & 59.93          & \textbf{29.90} & \textbf{29.88} & \textbf{29.55} & \textbf{26.14} & \textbf{27.93} & \textbf{25.43} & \textbf{24.72} & \textbf{25.16} & \textbf{30.03} & \textbf{23.72} \\
TRADES ($\lambda=1$)        & 60.18          & 28.93          & 28.91          & 29.12          & 25.79          & 27.07          & 25.00          & 23.65          & 24.31          & 28.76          & 23.22          \\
FAT           & 61.71          & 22.93          & 22.87          & 22.64          & 23.45          & 24.78          & 24.91          & 20.56          & 23.16          & 26.37          & 20.01          \\
IAT           & 57.04          & 21.40          & 21.39          & 22.37          & 19.18          & 19.63          & 18.92          & 15.50          & 16.63          & 23.26          & 15.50          \\
RST           & 60.30          & 23.56          & 23.61          & 23.71          & 22.40          & 24.69          & 24.18          & 21.66          & 23.82          & 27.05          & 21.18          \\ \hline\hline
Generalist        & \textbf{62.97} & \textbf{29.48} & \textbf{29.49} & \textbf{30.35} & \textbf{27.77} & \textbf{27.45} & \textbf{27.42} & \textbf{24.04} & \textbf{25.54} & \textbf{31.41} & \textbf{23.96}
\end{tabular}
\end{table*}

\subsection{Computational Cost and Tradeoff Comparison of Generalist}
\label{apd:a4}
We compute the actual training time of TRADES and Generalist (serial/parallel version) using ResNet-18 on RTX 3090 GPU in Table \ref{tab:cost}. We also report the standard deviations over 5 runs to show the sensitivity of Generalist. Neither version of Generalist is slower than TRADES. Generalist does perform both NT and naive AT, but the cost of NT is negligible so the overhead (NT+AT) is smaller than TRADES. 

Besides, Table \ref{tab:cost} delivers another important message. For the tradeoff between robustness and accuracy, it is hard to obtain acceptable robustness while maintaining clean accuracy above 89\% in the joint training framework (TRADES). For every percentage point increase in clean accuracy, the robust accuracy will decrease dramatically (e.g. TRADES can meet 89\% on clean accuracy but its robustness against APGD will drop to 30\%).
\renewcommand{\arraystretch}{0.85}
\begin{table}[!ht]
\caption{Evaluation of time complexity of our algorithm with different training methods using ResNet-18.}\label{tab:cost}
\centering
\begin{tabular}{l|ccc|c}
Method            & NAT                                                       & PGD100                                                     & APGD                                                        & \begin{tabular}[c]{@{}c@{}}Training \\ Time (mins)\end{tabular} \\ \hline\hline
TRADES            & \begin{tabular}[c]{@{}c@{}}89.91\\ $\pm$0.69\end{tabular} & \begin{tabular}[c]{@{}c@{}}34.25\\ $\pm$0.56\end{tabular} & \begin{tabular}[c]{@{}c@{}}30.20\\ $\pm$0.81\end{tabular} & 414                                                             \\ \hline
Generalist (Serial)   & \begin{tabular}[c]{@{}c@{}}89.11\\ $\pm$0.23\end{tabular} & \begin{tabular}[c]{@{}c@{}}50.12\\ $\pm$0.12\end{tabular} & \begin{tabular}[c]{@{}c@{}}46.12\\ $\pm$0.11\end{tabular} & 397                                                             \\ \hline
Generalist (Parallel) & \begin{tabular}[c]{@{}c@{}}89.09\\ $\pm$0.34\end{tabular} & \begin{tabular}[c]{@{}c@{}}50.00\\ $\pm$0.44\end{tabular} & \begin{tabular}[c]{@{}c@{}}46.53\\ $\pm$0.3\end{tabular}  & \textbf{342}                                                           
\end{tabular}
\end{table}

\subsection{Influence of Learning Rate}
\label{apd:a5}
In this part, we also study the influence of the learning rate for different distribution-aware tasks. For simplicity, we set $t^{\prime}$, $\gamma$ and $c$ as their best options according to the main body of the paper. We search the most grid of learning rate configurations in the range of {0.1, 0.01, 0.001} for both natural training and adversarial training. Our Generalist achieves its best and second-best natural accuracy when the learning rate for the clean learner is set to 0.1. And the optimal learning rate for robust accuracy is 0.01. Based on all the observations from \cref{tab:lr}, the learning pace of learners is a little different but the process is compatible.
\begin{table}[!ht]
\caption{Clean and robust accuracy (\%) on CIFAR-10 dataset using ResNet-18 with different learning rates.}\label{tab:lr}
\centering
\begin{tabular}{l|cc}
                 & NAT   & AA    \\ \hline\hline
NT=0.1, AT=0.01  & 89.09 & 46.37 \\
NT=0.1, AT=0.1   & 90.12 & 41.86 \\
NT=0.1, AT=0.001 & \textbf{90.45} & 43.55 \\
NT=0.01, AT=0.01 & 88.4  & \textbf{48.03} \\
NT=0.01, AT=0.1  & 88.25 & 42.98
\end{tabular}
\end{table}

\section{Proofs of Theoretical Results}
\subsection{Proof of Claim in Section 3.3}
\label{apd:b1}
\label{apd:global}
\begin{proof}
At epoch $t$, the parameters of the global learner are distributed to the experts and each expert train from this initialization with $c$ steps by calculating the gradients (e.g. using SGD optimizer). Following \cite{DBLP:journals/corr/abs-1803-02999}, we approximate the update performed by the initialization based on the Taylor expansion:
\begin{equation}
\label{eqn:proof_taylor}
\begin{aligned}
g^{t+c}=\ell^{\prime}\left(\boldsymbol\theta^{t+c}\right) &=\ell^{\prime}\left(\boldsymbol\theta^{t}\right)+\ell^{\prime \prime}\left(\boldsymbol\theta^{t}\right)\left(\boldsymbol\theta^{t+c}-\boldsymbol\theta^{t}\right)+O\left(\left\|\boldsymbol\theta^{t+c}-\boldsymbol\theta^{t}\right\|^{2}\right)\\
&\left.=\bar{g}^{t}+\bar{H}^{t}\left(\boldsymbol\theta^{t+c}-\boldsymbol\theta^{t}\right)+O\left(\tau^{2}\right)\right) \\
&=\bar{g}^{t}-\tau \bar{H}^{t} \sum_{j=t}^{t+c} g^{j}+O\left(\tau^{2}\right) \\
&=\bar{g}^{t}-\tau \bar{H}^{t} \sum_{j=t}^{t+c} \bar{g}^{j}+O\left(\tau^{2}\right).
\end{aligned}
\end{equation}
Recalling that $\mathcal{Z}^{i}$ represents an optimizer that updates the parameter vector at the $t$-th step: $\mathcal{Z}^{i}(\boldsymbol\theta,\tau)=\boldsymbol\theta-\tau\ell^{\prime}(\boldsymbol\theta)$. For each base-learner, we approximate the gradient at intervals: 
\begin{equation}
\begin{aligned}
g_{val}=\frac{\partial}{\partial \boldsymbol\theta^{t}} \ell\left(\boldsymbol\theta^{t+c}\right) &=\frac{\partial}{\partial \boldsymbol\theta^{t}} \ell\left(\mathcal{Z}^{t+c-1}\left(\mathcal{Z}^{t+c-2}\left(\ldots\left(\mathcal{Z}^{t}\left(\boldsymbol\theta^{t}\right)\right)\right)\right)\right) \\
&={\mathcal{Z}^{\prime}}^{t}\left(\boldsymbol\theta^{t}\right) \cdots {\mathcal{Z}^{\prime}}^{t+c-1}\left(\boldsymbol\theta^{t+c-1}\right) \ell^{\prime}\left(\boldsymbol\theta^{t+c}\right) \\
&=\left(I-\tau \ell^{\prime \prime}\left(\boldsymbol\theta^{t}\right)\right) \cdots\left(I-\tau \ell^{\prime \prime}\left(\boldsymbol\theta^{t+c-1}\right)\right) \ell^{\prime}\left(\boldsymbol\theta^{t+c}\right) \\
&=\left(\prod_{j=t}^{t+c-1}\left(I-\tau \ell^{\prime \prime}\left(\boldsymbol\theta^{j}\right)\right)\right) g^{t+c}.
\end{aligned}
\end{equation}
Replacing $\ell^{\prime \prime}\left(\boldsymbol\theta^{j}\right)$ with $\bar{H}^{j}$ and substituting $g^{t+c}$ for Eq. \ref{eqn:proof_taylor}, we expand to leading order:
\begin{equation}
\begin{aligned}
g_{val}&=\left(\prod_{j=t}^{t+c-1}\left(I-\tau \bar{H}^{j}\right)\right)\left(\bar{g}^{t+c}-\tau \bar{H}^{t+c} \sum_{j=t}^{t+c-1} \bar{g}^{j}\right)+O\left(\tau^{2}\right) \\
&=\left(I-\tau \sum_{j=t}^{t+c-1} \bar{H}^{j}\right)\left(\bar{g}^{t+c}-\tau \bar{H}^{t+c} \sum_{j=t}^{t+c-1} \bar{g}^{j}\right)+O\left(\tau^{2}\right) \\
&=\bar{g}^{t+c}-\tau \sum_{j=t}^{t+c-1} \bar{H}^{j} \bar{g}^{t+c}-\tau \bar{H}^{t+c} \sum_{j=t}^{t+c-1} \bar{g}^{j}+O\left(\tau^{2}\right)
\end{aligned}
\end{equation}
Therefore, we take the expectation of $g_{val}$ over steps, and obtain:
\begin{equation}
\mathbb{E}\left[g_{val}\right]=\mathbb{E}\left[\bar{g}^{t+c}\right]-\tau\mathbb{E}\left[\sum_{j=t}^{t+c-1} \bar{H}^{j} \bar{g}^{t+c}-\bar{H}^{t+c} \sum_{j=t}^{t+c-1} \bar{g}^{j}\right]+\mathbb{E}\left[O\left(\tau^{2}\right)\right]
\end{equation}
Recalling that $\boldsymbol\theta_{g}$ is mixed by $\boldsymbol\theta_n$ and $\boldsymbol\theta_{r}$. For simplicity of exposition, we use $p$ and $q$ to stand for the scalar factors, meaning $\boldsymbol\theta_{g}=p\boldsymbol\theta_{n}+q\boldsymbol\theta_{r}$. Ignoring the higher order terms, for each expert initialized by the global learner (e.g. $\boldsymbol\theta_{n}$), we have:
\begin{equation}
\begin{aligned}
\boldsymbol\theta_{n}=\boldsymbol\theta_{g}-\mathbb{E}_n\left[g_{val}\right]&=p\boldsymbol\theta_{n}+q\boldsymbol\theta_{r}-[\mathbb{E}\left[\bar{g}^{t+c}_n\right]+\tau_n\mathbb{E}\left[\sum_{j=t}^{t+c-1} \bar{H}^{j} \bar{g}^{t+c}_n-\bar{H}^{t+c} \sum_{j=t}^{t+c-1} \bar{g}^{j}_n\right]] \\
&=[p\boldsymbol\theta_{n}-\mathbb{E}\left[\bar{g}^{t+c}_n\right]] + [q\boldsymbol\theta_{r}-\tau_n\mathbb{E}\left[\bar{H}^{t+c} \sum_{j=t}^{t+c-1} \bar{g}^{j}_n-\sum_{j=t}^{t+c-1} \bar{H}^{j} \bar{g}^{t+c}_n\right]] \\
&=[p\boldsymbol\theta_{n}-\sum_{i=t}^{t+c-1} \bar{g}^{i}] + [q\boldsymbol\theta_{r}-\tau_n\sum_{i=t}^{t+c-1} \sum_{j=1}^{i-1} \bar{H}^{i} \bar{g}^{j}] \quad\quad (\text{for} \ c\ge2).
\end{aligned}
\end{equation}
The first term pushes $\theta_{n}$ to move forward the minimum of its assigned loss over its data distribution; while the second one improves generalization by increasing the inner product between gradients of different mini-batches and updating the parameters from the other task.
\end{proof}

\subsection{Proof of Theorem 1}
\label{apd:b2}
Before we present the proof of the Theorem we present useful intermediate results which we require in our proof.
\begin{proposition}
\label{pro:1}
Consider a sequence of loss functions ${\ell_a: \Theta\mapsto [0, 1]}_{a\in \mathcal{A}}$ drawn i.i.d. from some distribution $\mathcal{L}$ is given to an algorithm that generates a sequence of hypotheses $\left\{\boldsymbol\theta_{a} \in \Theta\right\}_{a \in\mathcal{A}}$ then the following inequality each hold w.p. $1-\delta$:
\begin{equation}
\frac{1}{T} \sum_{t=1}^{T} \underset{\ell \sim D}{\mathbb{E}} \ell\left(\boldsymbol\theta^{t}\right) \leq \frac{1}{T} \sum_{t=1}^{T} \ell^{t}\left(\boldsymbol\theta^{t}\right)+\sqrt{\frac{2}{T}\log \frac{1}{\delta}}.
\end{equation}
\end{proposition}
\begin{proof}
The proof of the Proposition can be directly derived from the Proposition 1 in \cite{DBLP:journals/tit/Cesa-BianchiCG04}.
\end{proof}
Then we could immediately obtain the below inequality by the symmetric version of the Azuma-Hoeffding inequality \cite{Azuma1967WEIGHTEDSO}
\begin{remark}
\label{remark1}
\begin{equation}
\frac{1}{T} \sum_{t=1}^{T} \underset{\ell \sim \mathcal{L}}{\mathbb{E}} \ell\left(\boldsymbol\theta^{t}\right) \geq \frac{1}{T} \sum_{t=1}^{T} \ell^{t}\left(\boldsymbol\theta^{t}\right)-\sqrt{\frac{2}{T}\log \frac{1}{\delta}}.
\end{equation}
\end{remark}
Finally, we give the definition of the regret of minimizing any subproblem:
\begin{definition}
(\textbf{Subproblem Regret}) Consider an algorithm generates the trajectory of states $\left\{\boldsymbol\theta^{t} \in \Theta\right\}_{t \in[T]}$, the regret of such an algorithm on loss function $\left\{\ell^{t}\right\}_{t \in[T]}$ is:
\begin{equation}
\bar{\mathbf{R}}=\sum_{t=1}^{T} \ell^{t}\left(\boldsymbol\theta^{t}\right)-\inf _{\boldsymbol\theta^{\star} \in \Theta} \sum_{t=1}^{T} \ell^{t}(\boldsymbol\theta).
\end{equation}
\end{definition}
\begin{theorem}
\label{thm:1}
(Restated) Consider an algorithm with regret bound $R_{T}$ that generates the trajectory of states for two base learners, for any parameter state $\boldsymbol\theta \in \Theta$, given a sequence of convex surrogate evaluation functions ${\ell: \Theta\mapsto [0, 1]_{a\in \mathcal{A}}}$ drawn i.i.d. from some distribution $\mathcal{L}$, the expected error of the global learner $\boldsymbol\theta_{g}$ on both tasks over the test set can be bounded with probability at least $1-\delta$:
\begin{equation}
\underset{\ell \sim \mathcal{L}}{\mathbb{E}} \ell\left(\boldsymbol\theta_{g}\right) \leq \underset{\ell \sim \mathcal{L}}{\mathbb{E}} \ell\left(\boldsymbol\theta\right)+\frac{\mathbf{R}_{T}}{T}+2\sqrt{\frac{2}{T}\log \frac{1}{\delta}}.
\end{equation}
\end{theorem}
\begin{proof}
From Theorem \ref{pro:1} and Remark \ref{remark1}, we obtain that
\begin{equation}
\frac{1}{T} \sum_{t=1}^{T} \underset{\ell \sim \mathcal{L}}{\mathbb{E}} \ell\left(\boldsymbol\theta^{t}\right) \leq \frac{1}{T} \sum_{t=1}^{T} \ell^{t}\left(\boldsymbol\theta\right)+\frac{\bar{\mathbf{R}}}{T}+\sqrt{\frac{2}{T}\log \frac{1}{\delta}} \leq \underset{\ell \sim \mathcal{L}}{\mathbb{E}} \ell\left(\boldsymbol\theta\right)+\frac{\bar{\mathbf{R}}}{T}+2\sqrt{\frac{2}{T}\log \frac{1}{\delta}}.
\end{equation}
It is obvious that:
\begin{equation}
\frac{\bar{\mathbf{R}}}{T}+\sqrt{\frac{2}{T}\log \frac{1}{\delta}} \leq \frac{\mathbf{R}_{T}}{T}+\sqrt{\frac{2}{T}\log \frac{1}{\delta}} \quad \text { and } \quad \frac{\bar{\mathbf{R}}}{T}+2\sqrt{\frac{2}{T}\log \frac{1}{\delta}} \leq \frac{\mathbf{R}_{T}}{T}+2\sqrt{\frac{2}{T}\log \frac{1}{\delta}}.
\end{equation}
So we obtain:
\begin{equation}
\frac{1}{T} \sum_{t=1}^{T} \underset{\ell \sim \mathcal{L}}{\mathbb{E}} \ell\left(\boldsymbol\theta^{t}\right) \leq \frac{1}{T} \sum_{t=1}^{T} \ell^{t}\left(\boldsymbol\theta\right)+\frac{\mathbf{R}_{T}}{T}+\sqrt{\frac{2}{T}\log \frac{1}{\delta}} \leq \underset{\ell \sim \mathcal{L}}{\mathbb{E}} \ell\left(\boldsymbol\theta\right)+\frac{\mathbf{R}_{T}}{T}+2\sqrt{\frac{2}{T}\log \frac{1}{\delta}}.
\end{equation}
Recalling that in Section 3.3, $\boldsymbol\theta_{g}$ can be expressed by the linear combination of $\boldsymbol\theta_{n}$ and $\boldsymbol\theta_{r}$ through $t=1,\cdots,T$ since $\boldsymbol\theta_{g}$ is aggregateed by EMA, so the above inequality can be further derived by the Jensen's inequality (convex surrogate functions could be selected to evaluate the test errors instead of the 0-1 loss):
\begin{equation}
\begin{aligned}
\underset{\ell \sim \mathcal{L}}{\mathbb{E}} \ell\left(\boldsymbol\theta_{g}\right)=\underset{\ell \sim \mathcal{L}}{\mathbb{E}} \ell\left(\sum_{t=1}^{T}\boldsymbol\theta^{t}\right) \leq \frac{1}{T} \sum_{t=1}^{T} \underset{\ell \sim \mathcal{L}}{\mathbb{E}} \ell\left(\boldsymbol\theta^{t}\right) &\leq \frac{1}{T} \sum_{t=1}^{T} \ell^{t}\left(\boldsymbol\theta\right)+\frac{\mathbf{R}_{T}}{T}+\sqrt{\frac{2}{T}\log \frac{1}{\delta}} \\
&\leq \underset{\ell \sim \mathcal{L}}{\mathbb{E}} \ell\left(\boldsymbol\theta\right)+\frac{\mathbf{R}_{T}}{T}+2\sqrt{\frac{2}{T}\log \frac{1}{\delta}}.
\end{aligned}
\end{equation}
Note that this inequality also holds when applying weight averaging technique to the base-learner, because weight averaging is still the linear combination of all history states.
\end{proof}

\end{document}